\documentclass{article}

\PassOptionsToPackage{numbers,compress}{natbib}
\usepackage[final]{neurips_2019}

\usepackage[utf8]{inputenc} 
\usepackage[T1]{fontenc}    
\usepackage{url}            
\usepackage{booktabs}       
\usepackage{amsfonts}       
\usepackage{nicefrac}       
\usepackage{microtype}      
\usepackage{multirow}
\usepackage{array}
\usepackage{makecell}
\usepackage{graphicx}
\usepackage{wrapfig}
\usepackage{amsmath}
\usepackage{amsthm}
\usepackage{enumitem}

\usepackage[table,xcdraw]{xcolor}
\definecolor{aquamarine}{rgb}{0.5, 1.0, 0.83}
\definecolor{brilliantlavender}{rgb}{0.96, 0.73, 1.0}
\definecolor{jonquil}{rgb}{0.98, 0.85, 0.37}
\definecolor{celadon}{rgb}{0.67, 0.88, 0.69}

\makeatletter
\newcommand\footnoteref[1]{\protected@xdef\@thefnmark{\ref{#1}}\@footnotemark}
\newtheorem{theorem}{Theorem}[]
\makeatother

\title{Attention over Parameters \\ for Dialogue Systems}

\author{%
Andrea Madotto, Zhaojiang Lin, Chien-Sheng Wu, Jamin Shin, Pascale Fung \\
  Centre for Artificial Intelligence Research (CAiRE)\\
  The Hong Kong University of Science and Technology\\
  Clear Water Bay, Hong Kong \\
  {\tt [amadotto,zlinao,cwuak,jmshinaa]@connect.ust.hk } \\
  {\tt pascale@ece.ust.hk}
}

\begin{document}

\maketitle
\begin{abstract}
Dialogue systems require a great deal of different but complementary expertise to assist, inform, and entertain humans. For example, different domains (e.g., restaurant reservation, train ticket booking) of goal-oriented dialogue systems can be viewed as different skills, and so does ordinary chatting abilities of chit-chat dialogue systems. In this paper, we propose to learn a dialogue system that independently parameterizes different dialogue skills, and learns to select and combine each of them through \textit{Attention over Parameters} (\textit{AoP}). The experimental results show that this approach achieves competitive performance on a combined dataset of MultiWOZ~\citep{budzianowski2018multiwoz}, In-Car Assistant~\citep{ericKVR2017}, and Persona-Chat~\citep{personachat}. Finally, we demonstrate that each dialogue skill is effectively learned and can be combined with other skills to produce selective responses. 
\end{abstract}

\section{Introduction}
Unlike humans who can do both, goal-oriented dialogues~\citep{williams2007partially,young2013pomdp} and chit-chat conversations~\citep{serban2016generative,vinyals2015neural} are often learned with separate models. A more desirable approach for the users would be to have a single chat interface that can handle both casual talk and tasks such as reservation or scheduling. This can be formulated as a problem of learning different conversational skills across multiple domains. A skill can be either querying a database, generating daily conversational utterances, or interacting with users in a particular task-domain (e.g. booking a restaurant). One challenge of having multiple skills is that existing datasets either focus only on chit-chat or on goal-oriented dialogues. This is due to the fact that traditional goal-oriented systems are modularized~\citep{williams2007partially,hori2009statistical,lee2009example,levin2000stochastic,young2013pomdp}; thus, they cannot be jointly trained with end-to-end architecture as in chit-chat. 
However, recently proposed end-to-end trainable models~\citep{eric-manning:2017:EACLshort,wu2019global,reddy2018multi,yavuzdeepcopy} and datasets~\citep{bordes2016learning,ericKVR2017} allow us to combine goal-oriented~\citep{budzianowski2018multiwoz,ericKVR2017} and chit-chat~\citep{personachat} into a single benchmark dataset with multiple conversational skills as shown in Table~\ref{Example-table}.

A straight forward solution would be to have a single model for all the conversational skills, which has shown to be effective to a certain extent by~\citep{zhao2017generative} and \citep{mccann2018natural}. Putting aside the performance in the tasks, such fixed shared-parameter framework, without any task-specific designs, would lose controllability and interpretability in the response generation. In this paper, instead, we propose to model multiple conversational skills using the \textit{Mixture of Experts} (MoE)~\citep{jacobs1991adaptive} paradigm, i.e., a model that learns and combine independent specialized experts using a gating function. 
\begin{table*}[t]
\centering

\begin{tabular}{rlc}
\hline
\multicolumn{1}{c}{\textbf{Spk.}} & \multicolumn{1}{c}{\textbf{Conversation}} & \textbf{Skills} \\ \hline
\textit{Usr:} & Can you help me find a cheap 2 star hotel? & \multicolumn{1}{l}{} \\
\rowcolor[HTML]{EFEFEF} 
\textit{Sys:} & \begin{tabular}[c]{@{}l@{}}\texttt{{\color[HTML]{3166FF} SELECT} * {\color[HTML]{3166FF} FROM} hotel {\color[HTML]{3166FF} WHERE} pricerange=`cheap'}\\ \texttt{{\color[HTML]{3166FF} AND} stars=2 {\color[HTML]{3166FF} AND} type=`hotel'}\end{tabular} & {\color[HTML]{FD6864} SQL+HOTEL} \\
\rowcolor[HTML]{EFEFEF} 
\textit{Mem:} & (Result table from the SQL query, with \textbf{real} entities.) &  \\
\rowcolor[HTML]{EFEFEF} 
\textit{Sys:} & \begin{tabular}[c]{@{}l@{}}Arms is a 2 star hotel in the centre. Shall I book a room for you?\end{tabular} & {\color[HTML]{32CB00} HOTEL} \\
\textit{Usr:} & Yes, thank you! &  \\
\rowcolor[HTML]{EFEFEF} 
\textit{Sys:} & Sure, for how many people should I book it for? & {\color[HTML]{32CB00} HOTEL} \\
\multicolumn{1}{l}{} &  & \multicolumn{1}{l}{} \\
\multicolumn{3}{c}{$\cdots$} \\
\multicolumn{1}{l}{} &  & \multicolumn{1}{l}{} \\
\rowcolor[HTML]{EFEFEF} 
\textit{Sys:} & Hello! How are you today? & {\color[HTML]{9698ED} PERSONA} \\
\textit{Usr:} & I am good thank you! Do you like to ski? I really like it. &  \\
\rowcolor[HTML]{EFEFEF} 
\textit{Sys:} & No, I like to code in Haskell. Do you like to code? & {\color[HTML]{9698ED} PERSONA} \\ \hline
\end{tabular}
\caption{An example from the dataset which includes both chit-chat and task-oriented conversations. The model has to predict all the \textit{Sys} turn, which includes SQL query and generating response from a the Memory content, which is dynamically updated with the queries results. The skills are the prior knowledge needed for the response, where Persona refers to chit-chat.}
\label{Example-table}
\end{table*}
For instance, each expert could specialize in different dialogues domains (e.g., Hotel, Train, Chit-Chat etc.) and skills (e.g., generate SQL query). A popular implementation of MoE \citep{shazeer2017outrageously,kaiser2017one} uses a set of linear transformation (i.e., experts) in between two LSTM \citep{schmidhuber:1987:srl} layers. However, several problems arise with this implementation: 1) the model is computationally expensive as it has to decode multiple times each expert and make the combination at the representation-level; 2) no prior knowledge is injected in the expert selection (e.g., domains); 3) Seq2Seq model has limited ability in extracting information from a Knowledge Base (KB) (i.e., generated by the SQL query)~\citep{ericKVR2017}, as required in end-to-end task-oriented dialogues systems~\citep{bordes2016learning}. The latter can be solved by using more advanced multi-hop models like the Transformer~\citep{vaswani2017attention}, but the remaining two need to be addressed. Hence, in this paper we: 
\begin{itemize}[leftmargin=*]
    \item propose a novel Transformer-based architecture called \textit{Attention over Parameters} (AoP). This model parameterize the conversational skills of end-to-end dialogue systems with independent decoder parameters (experts), and learns how to dynamically select and combine the appropriate decoder parameter sets by leveraging prior knowledge from the data such as domains and skill types; 
    \item proof that AoP is algorithmically more efficient compared to forwarding all the Transformer decoders and then mix their output representation, like is normally done in MoE. Figure~\ref{fig:aopdiagram} illustrates the high-level intuition of the difference;
    %
    \item empirically show the effectiveness of using specialized parameters in a combined dataset of MultiWOZ~\citep{budzianowski2018multiwoz}, In-Car Assistant~\citep{ericKVR2017}, and Persona-Chat~\citep{personachat}, which to the best of our knowledge, is the first evaluation of this genre i.e. end-to-end large-scale multi-domains/skills. Moreover, we show that our model is highly interpretable and is able to combine different learned skills to produce compositional responses.
\end{itemize}


\section{Methodology}
We use the standard encoder-decoder architecture and avoid any task-specific designs~\citep{wu2019global,reddy2018multi}, as we aim to build a generic conversation model for both chit-chat and task-oriented dialogues. More specifically, we use a Transformer for both encoder and decoder.

Let us define the sequence of tokens in the dialogue history as $D=\{ d_1,\dots,d_m\}$ and the dynamic memory content as a sequence of tokens $M=\{ m_1,\dots,m_{z}\}$. The latter can be the result of a SQL query execution (e.g., table) or plain texts (e.g., persona description), depending on the task. The dialogue history $D$ and the memory $M$ are concatenated to obtain the final input denoted by $X=[D;M] =\{ x_1,\dots,x_{n=m+z}\}$. We then denote $Y=\{ y_1,\dots,y_k\}$ as the sequence of tokens that the model is expected to produce. Without loss of generality, $Y$ can be both plain text and SQL-like queries. Hence, the model has to learn when to issue database queries and when to generate human-like responses. Finally, we define a binary skill vector $V=\{ v_1,\dots,v_{r} \}$ that specifies the type of skills required to generate $Y$. This can be considered as a prior vector for learning to select the correct expert during the training\footnote{the vector $V$ will be absent during the testing}. For example, in Table~\ref{Example-table} the first response is of type SQL in the Hotel domain, thus the skill vector $V$ will have $v_{\text{SQL}}=1$ and $v_{\text{Hotel}}=1$, while all the other skill/domains are set to zero~\footnote{With the assumption that at each index in $V$ is assigned a semantic skill (e.g. SQL position $i$)}. More importantly, we may set the vector $V$ to have multiple ones to enforce the model to compose skills to achieve a semantic compositionality of different experts.

 \begin{figure*}[t]
    \centering
    \includegraphics[width=0.92\linewidth]{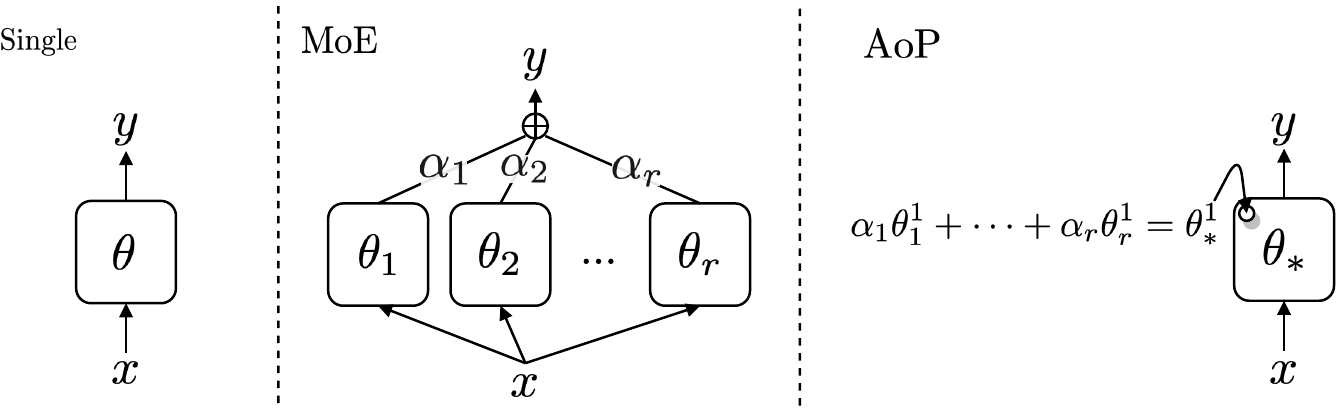}
        \caption{Comparisons between Single model, Mixture of Experts (MoE)~\citep{jacobs1991adaptive}, and Attention over Parameters (\textit{AoP}).}
    \label{fig:aopdiagram}
\end{figure*}
\subsection{Encoder-Decoder}
To map the input sequence to the output sequence, we use a standard Transformer~\citep{vaswani2017attention} and denote the encoder and decoder as TRS$_{enc}$ and TRS$_{dec}$, respectively. The input of a Transformer is the embedded representation of the input words; thus, we define a word embedding matrix $E\in \mathbb{R}^{d \times |V|}$ where $d$ is the embedding size and $|V|$ is the cardinality of the vocabulary. The input $X$, with its positional embedding (Appendix A1 for more information), are encoded as the following equation:
\begin{equation}
    H = \textrm{TRS}_{enc}(E(X)) , 
\end{equation}
where $H\in \mathbb{R}^{d_{model} \times n}$, and E. Then the decoder receives the target sequence shifted by one $Y_{:k-1} = \{\text{<SOS>}, y_1, \dots, y_k\}$ as the input. Using teacher-forcing~\citep{williams1989learning}, the model is trained to produce the correct sequence $Y$. The output of the decoder is produced as follow:
\begin{equation}
    O = \textrm{TRS}_{dec}(E(Y_{:k-1}),H) , \label{eq_dec}
\end{equation}
where $O\in \mathbb{R}^{d_{model} \times k}$. Finally, a distribution over the vocabulary is generated for each token by an affine transformation $W\in \mathbb{R}^{d_{model} \times |V|}$ followed by a Softmax function. 
\begin{equation}
    P(Y|X) = \textrm{Softmax}(O^TW)  ,
\end{equation}
In addition, $P(Y|X)$ is mixed with the encoder-decoder attention distribution to enable to copy token from the input sequence as in~\citep{see-liu-manning:2017:Long}. The model is then trained to minimize a standard cross entropy loss function and at inference time to generate one token at the time in an auto-regressive manner~\citep{graves2013generating}. Hence, the training loss is defined as: 
\begin{equation}
\mathcal{L}_{P(Y|X)}=-\sum_{t=1}^{k}\log \left(P(Y|X)_t\left(y_{t}\right)\right).
\end{equation}

\subsection{Attention over Parameters}
\label{sec:aop}
The main idea is to produce a single set of parameters for decoder TRS$_{dec}$ by the weighted sum of $r$ independently parameterized decoders. This process is similar to attention~\citep{luong-pham-manning:2015:EMNLP} where the memories are the parameters and the query is the encoded representation. Let us define $\Theta=[ \theta_1,\dots,\theta_r ]$ as the list of parameters for $r$ decoders, since a TRS$_{dec}$ is represented by its parameters $\theta$. Since each $\theta$ can be sized in the order of millions, we assign the corresponding key vectors to each $\theta$, similar to key-value memory networks~\citep{miller2016key}. Thus, we use a key matrix $K\in \mathbb{R}^{d_{model} \times r}$ and a Recurrent Neural Networks (RNN), in this instance a GRU~\citep{cho2014learning}, to produce the query vector by processing the encoder output $H$.  The attention weights for each decoders' parameters is computed as follow:
\begin{align}
    q      &= \textrm{RNN}(H)\\
    \alpha &= \text{Softmax}(qK) 
    \label{eq_att}
\end{align}
where $q \in \mathbb{R}^{d_{model}}$ and $\alpha \in \mathbb{R}^r$ is the attention vectors where each $\alpha_i$ is the score corresponding to $\theta_i$. Hence, the new set of parameters is computed as follow:
\begin{equation}
    \theta^* = \sum_i^r \alpha_i \theta_i \label{eq:aop}
\end{equation}
The combined set of parameters $\theta^*$ are then used to initialize a new TRS$_{dec}$, and Equation~\ref{eq_dec} will be applied to the input based on this. Equation~\ref{eq_att} is similar to the gating function proposed in~\citep{shazeer2017outrageously, jacobs1991adaptive}, but the resulting scoring vector $\alpha$ is applied directly to the parameter instead of the output representation of each decoder, holding an algorithmically faster computation. 

\begin{theorem}
The computation cost of Attention over Parameters (AoP) is always lower than Mixture Of Experts (MoE) for sequence longer than 1.
\end{theorem}
\begin{proof}
Let $f_\theta: \mathbb{R}^d \rightarrow \mathbb{R}^n$ a generic function parametrized by $\theta$. Without loss of generality, we define $\theta$ as a affine transformation $W \in  \mathbb{R}^{d \times n}$. Let $X \in  \mathbb{R}^{t \times d}$ a generic input sequence of length $t$ and $d$ dimensional size. Let the set $F=[f_{\theta_1},\cdots,f_{\theta_r}]$ be the set of $r$ experts. Hence, the operation done by MoE are:
\begin{equation}
    \textrm{MoE}(X)= f_{\theta_1}(X) + \cdots + f_{\theta_r}(X)= XW_1 + \cdots + XW_r
\end{equation}
Thus the computational cost in term of operation is $\mathcal{O}(rtdn + rtn)$ since the cost of $ f_{\theta_i}(X)$ is $\mathcal{O}(tdn)$ and it is repeated $r$ times, and the cost of summing the representation is $\mathcal{O}(rtn)$. On the other hand, the operation done by AoP are:
\begin{gather}
    \theta^* = \theta_1 + \cdots + \theta_r = W_1 + \cdots + W_r\\
    \textrm{AoP}(X) = f_{\theta^*}(X) =  XW^*
\end{gather}
in this case the computational cost in term of operation is $\mathcal{O}((r+t)dn)$ since the cost of summing the parameters is $\mathcal{O}(rdn)$ and the cost of $f_{\theta^*}$ is $\mathcal{O}(tdn)$. Hence, it is easy to verify that if $t>1$ then:
\begin{equation}
    rtdn + rtn \geq (rt)dn \geq (r+t)dn
\end{equation}
Furthermore, the assumption of using a simple affine transformation $W$ is actually an optimal case. Indeed, assuming that the cost of parameters sum is equal to the number of operation is optimistic, for instance already by using attention the number of operations increases but the number of parameters remains constant.  
\end{proof}

Importantly, if we apply $\alpha$ to each of the output representation $O^i$ generated by the TRS$_{dec}^i$, we end up having a Transformer-based implementation of MoE. We call this model as Attention over Representation (AoR). Finally, an additional loss term is used to supervise the attention vector $\alpha$ by using the prior knowledge vector $V$. Since multiple decoder parameters can be selected at the same time, we use a binary cross-entropy to train each $\alpha_i$. Thus a second loss is defined as: 
\begin{equation}
\mathcal{L}_{V} =- \sum_{i=1}^{r} v_{i} \times \text{log} \sigma(qK)_{i} + (1-v_{i}) \times \text{log} (1- \sigma(q K)_{i}) \nonumber
\end{equation}
The final loss is the summation of $\mathcal{L}_{P(Y|X)}$ and $\mathcal{L}_{V}$. 

Finally, in AoP, but in general in the MoE framework, stacking multiple layers (e.g., Transformer) leads to models with a large number of parameters, since multiple experts are repeated across layers. An elegant workaround is the Universal Transformer~\citep{dehghani2018universal}, which loops over an \textit{unique} layer and, as shown by~\citep{dehghani2018universal}, holds similar or better performance than a multi-layer Transformer. In our experiment, we report a version of AoP that uses this architecture, which for instance does not add any further parameter to the model.

\section{Experiments and Results}
\subsection{Dataset} 
To evaluate the performance of our model for different conversational skills, we propose to combine three publicly available datasets: MultiWOZ~\citep{budzianowski2018multiwoz}, Stanford Multi-domain Dialogue~\citep{ericKVR2017} and Persona-Chat~\citep{personachat}. 

\begin{wrapfigure}{r}{0.5\textwidth}
\caption{Datasets statistics}
\label{data_stat}
\begin{tabular}{r|ccc}
\hline
\multicolumn{1}{l|}{} & \textbf{SMD} & \textbf{MWOZ} & \textbf{Persona} \\ \hline
\textit{\#Dialogues} & 2425 & 8,438 & 12,875 \\ \hline
\textit{\#turns} & 12,732 & 115,424 & 192,690 \\ \hline
\textit{Avg. turns} & 5.25 & 13.68 & 14.97 \\ \hline
\textit{Avg. tokens} & 8.02 & 13.18 & 11.96 \\ \hline
\textit{Vocab} & 2,842 & 24,071 & 20,343 \\ \hline
\end{tabular}
\makeatletter\def\@captype{table}\makeatother
\end{wrapfigure}
\paragraph{MultiWOZ (MWOZ)} is a human-to-human multi-domain goal-oriented dataset annotated with dialogue acts and states. In this dataset, there are seven domains (i.e., Taxi, Police, Restaurant, Hospital, Hotel, Attraction, Train) and two APIs interfaces: SQL and BOOK. The former is used to retrieve information about a certain domain and the latter is used to book restaurants, hotels, trains, and taxis. 
We refine this dataset to include SQL/BOOK queries and their outputs using the same annotations schema as~\citep{bordes2016learning}.

Hence, each response can either be plain text conversation with the user or SQL/BOOK queries, and the memory is dynamically populated with the results from the queries as the generated response is based on such information. 
This transformation allows us to train end-to-end models that learns how and when to produce SQL queries, to retrieve knowledge from a dynamic memory, and to produce plain text response.
A detailed explanation is reported in Appendix A3, together with some samples. 

\paragraph{Stanford Multi-domain Dialogue (SMD)}is another human-to-human multi-domain goal-oriented dataset that is already designed for end-to-end training. There are three domains in this dataset (i.e., Point-of-Interest, Weather, Calendar). The difference between this dataset and MWOZ is that each dialogue is associated with a set of records relevant to the dialogues. The memory is fixed in this case so the model does not need to issue any API calls. However, retrieving the correct entities from the memory is more challenging as the model has to compare different alternatives among records.

\paragraph{Persona-Chat}is a multi-turn conversational dataset, in which two speakers are paired and different persona descriptions (4-5 sentences) are randomly assigned to each of them. For example, ``\textit{I am an old man}'' and ``\textit{I like to play football}'' are one of the possible persona descriptions provided to the system. Training models using this dataset results in a more persona consistent and fluent conversation compared to other existing datasets~\citep{personachat}. Currently, this dataset has become one of the standard benchmarks for chit-chat systems, thus, we include it in our evaluation.

For all three datasets, we use the training/validation/test split provided by the authors and we keep all the \textit{real entities} in input instead of using their delexicalized version as in~\citep{budzianowski2018multiwoz,ericKVR2017}. This makes the task more challenging, but at the same time more interesting since we force the model to produce real entities instead of generic and frequent placeholders. Table~\ref{data_stat} summarizes the dataset statistics in terms of number of dialogues, turns, and unique tokens. Finally, we merge the three datasets obtaining 154,768/19,713/19,528 for training, validation and, test respectively, and a vocabulary size of 37,069 unique tokens.

\subsection{Evaluation Metrics}
\paragraph{Goal-Oriented} For both MWOZ and SMD, we follow the evaluation done by existing works~\citep{eric-manning:2017:EACLshort,zhao2017generative,madotto2018mem2seq,wu2017dstc6}. We use BLEU\footnote{Using the \texttt{multi-bleu.perl} script} score~\citep{pAPIneniBLEU2002} to measure the response fluency and Entity F1-Score~\citep{wen2016network,zhao2017generative} to evaluates the ability of the model to generate relevant entities from the dynamic memory. Since MWOZ also includes SQL and BOOK queries, we compute the exact match accuracy (i.e., $ACC_{SQL}$ and $ACC_{BOOK}$) and BLEU score (i.e., $BLEU_{SQL}$ and $BLEU_{BOOK}$). Furthermore, we also report the F1-score for each domain in both MWOZ and SMD. 

\paragraph{Chit-Chat} We compare perplexity, BLEU score, F1-score~\citep{dinan2019second}, and Consistency score of the generate sentences with the human-generated prediction. The Consistency score is computed using a Natural Language Inference (NLI) model trained on dialogue NLI~\citep{dnli}, a recently proposed corpus based on Persona dataset. We fine-tune a pre-trained BERT model~\citep{devlin2018bert} using the dialogue DNLI corpus and achieve a test set accuracy of 88.43\%, which is similar to the best-reported model in~\citep{dnli}. The consistency score is defined as follow:
\begin{align}
    &\text{\textbf{NLI}}(u, p_j) = \Bigg\{
  \begin{array}{rcr}
    1 & \text{if $u$ entails $p_j$} \\
    0 & \text{if $u$ is independent to $p_j$} \\
    -1 & \text{if $u$ contradicts $p_j$} \\
  \end{array} \nonumber
  \\
  &\text{\textit{\textbf{C}}}(u) = \sum_j^{m} \text{\textbf{NLI}}(u, p_j)
\end{align}
where $u$ is a generated utterance and $p_j$ is one sentence in the persona description. In~\cite{dnli,madotto2019personalizing}, the authors showed that by re-ranking the beam search hypothesis using the DNLI score (i.e., \textit{\textbf{C}} score), they achieved a substantial improvement in dialogue consistency. Intuitively, having a higher consistency \textit{\textbf{C}} score means having a more persona consistent dialogue response.

\subsection{Baselines}
In our experiments, we compare Sequence-to-Sequence (\textit{Seq2Seq})~\citep{see-liu-manning:2017:Long}, Transformer (\textit{TRS})~\citep{vaswani2017attention}, Mixture of Expert (\textit{MoE})~\citep{shazeer2017outrageously} and Attention over Representation (\textit{AoR}) with our proposed Attention over Parameters (\textit{AoP}). In all the models, we used the same copy-mechanism as in~\citep{see-liu-manning:2017:Long}. In \textit{AoR} instead of mixing the parameters as in Equation \ref{eq:aop}, we mix the output representation of each transformer decoder (i.e.  Equation \ref{eq_dec}). For all \textit{AoP}, \textit{AoR}, and \textit{MoE}, $r=13$ is the number of decoders (experts): 2 skills of SQL and BOOK, 10 different domains for MWOZ+SMD, and 1 for Persona-Chat. Furthermore, we include also the following experiments: \textit{AoP} that uses the gold attention vector $V$, which we refer as \textit{AoP} w/ Oracle (or \textit{AoP} + O); \textit{AoP} trained by removing the $\mathcal{L}_{V}$ from the optimization (\textit{AoP w/o $\mathcal{L}_{V}$}); and as aforementioned, the Universal Transformer for both AoP (\textit{AoP + U}) and the standard Transformer (\textit{TRS + U}) (i.e., 6 hops). All detailed model description and the full set of hyper-parameters used in the experiments are reported in Appendix A4. 

\subsection{Results} 

Table~\ref{mwoz-smd} and Table~\ref{persona} show the respectively evaluation results in MWOZ+SMD and Persona-Chat datasets.

\begin{wrapfigure}{r}{0.5\textwidth}
\vspace{-20pt}
\caption{Results for the Persona-Chat dataset.} \label{persona}
\begin{tabular}{r|cccc}
\hline
\multicolumn{1}{c|}{\textbf{Model}} & \textbf{Ppl.} & \textbf{F1} & \textbf{C} & \textbf{BLEU} \\ \hline
\textit{Seq2Seq}         & 39.42  &	6.33  &	0.11  &	2.79  \\ \hline
\textit{TRS}             & 43.12  &	7.00  &	0.07  &	2.56  \\ \hline
\textit{MoE}             & \textbf{38.63}  & \textbf{7.33}  &	0.19  &	2.92  \\ \hline
\textit{AoR}            & 40.18   & 6.66  & 0.12  & 2.69  \\ \hline
\textit{AoP}       & 39.14  &	7.00  &	\textbf{0.21}  &	\textbf{3.06}  \\ \hline\hline   
\textit{TRS + U}         & 43.04  &	\textbf{7.33}  & 0.15  &	2.66  \\ \hline
\textit{AoP + U}         & \textbf{37.40}  &	7.00  &	\textbf{0.29}  &	\textbf{3.22}  \\ \hline \hline 
\textit{AoP w/o $\mathcal{L}_{V}$}       & 42.81  &	6.66  &	0.12  &	2.85  \\ \hline
\textit{AoP + O}         & \textit{40.16}  &	\textit{7.33}  &	\textit{0.21}  &	\textit{2.91}  \\ \hline 
\end{tabular}
\end{wrapfigure}
From Table~\ref{mwoz-smd}, we can identify four patterns. 1) \textit{AoP} and \textit{AoR} perform consistently better then other baselines which shows the effectiveness of combining parameters by using the correct prior $V$; 2) \textit{AoP} performs consistently, but marginally, better than \textit{AoR}, with the advantage of an algorithmic faster inference; 3) Using Oracle (\textit{AoP}+O) gives the highest performance in all the measures, which shows the performance upper-bound for \textit{AoP}. Hence, the performance gap when not using oracle attention is most likely due to the error in attention $\alpha$ (i.e., 2\% error rate). Moreover, Table~\ref{mwoz-smd} shows that by removing $\mathcal{L}_{V}$ (\textit{AoP w/o $\mathcal{L}_{V}$}) the model performance decreases, which confirms that good inductive bias is important for learning how to select and combine different parameters (experts). Additionally, in Appendix A5, we report the per-domain F1-Score for SQL, BOOK and sentences, and Table~\ref{persona} and Table~\ref{mwoz-smd} with the standard deviation among the three runs. 

Furthermore, from Table~\ref{persona}, we can notice that \textit{MoE} has the lowest perplexity and F1-score, but \textit{AoP} has the highest Consistency and BLUE score. Notice that the perplexity reported in~\citep{personachat} is lower since the vocabulary used in their experiments is smaller. In general, the difference in performance among different models is marginal except for the Consistency score; thus, we can conclude that all the models can learn this skill reasonably well. Consistently with the previous results, when $\mathcal{L}_{V}$ is removed from the optimization, the models' performance decreases. 
\begin{table}[t]
\caption{Results for the goal-oriented responses in both MWOZ and SMD. Last raw, and italicized, are the Oracle results, and bold-faced are best in each setting (w and w/o Universal). Results are averaged among three run (full table in Appendix A6).} \label{mwoz-smd}
\centering
\resizebox{0.9\textwidth}{!}{%
\begin{tabular}{r|cc|cc|cc}
\hline
\multicolumn{1}{c|}{\textbf{Model}} & \textbf{F1} & \textbf{BLEU} & \textbf{SQL$_{Acc}$} & \textbf{SQL$_{BLEU}$} & \textbf{BOOK$_{Acc}$} & \textbf{BOOK$_{BLEU}$} \\ \hline
\textit{Seq2Seq}        & 38.37  &	9.42     & 49.97  &	81.75  & 39.05  & 79.00  \\ \hline
\textit{TRS}            & 36.91  &	9.92   & 61.96  &	89.08  & 46.51  & 78.41  \\ \hline
\textit{MoE}            & 38.64  &	9.47    & 53.60  &	85.38  & 37.23  & 78.55  \\ \hline
\textit{AoR}        & 40.36  &	10.66  &	69.39  &	90.64  &	52.15  &	81.15              \\ \hline
\textit{AoP}  & \textbf{42.26}  &	\textbf{11.14}   &\textbf{71.1}  &	\textbf{90.90}  & \textbf{56.31}  & \textbf{84.08}  \\ \hline \hline
\textit{TRS + U}        & 39.39  &9.29	  & 61.80 &89.70	& 50.16 & 79.05  \\ \hline
\textit{AoP + U}        & \textbf{44.04}  &	\textbf{11.26}  &	\textbf{74.83}  &	\textbf{91.90}  &	\textbf{56.37}  &	\textbf{84.15}              \\ \hline \hline
\textit{AoP w/o $\mathcal{L}_{V}$}        & 38.50  &	10.50     & 61.47  &	88.28  & 52.61  & 80.34  \\ \hline
\textit{AoP+O}   & \textit{46.36}  &	\textit{11.99}   & \textit{73.41}  &	\textit{93.81}  & \textit{56.18} & \textit{86.42}  \\ \hline
\end{tabular}
}
\end{table}

Finally, in both Table~\ref{mwoz-smd} and Table~\ref{persona}, we report the results obtained by using the Universal Transformer, for both AoP and the Transformer. By adding the layer recursion, both models are able to consistently improve all the evaluated measures, in both Persona-Chat and the Task-Oriented tasks. Especially AoP, which achieves better performance than Oracle (i.e. single layer) in the SQL accuracy, and a consistently better performance in the Persona-Chat evaluation.



\section{Skill Composition}


To demonstrate the effectiveness of our model in learning independent skills and composing them together, we \textit{manually trigger} skills by modifying $\alpha$ and generate 14 different responses for the same input dialogue context. This experiment allows us to verify whether the model accurately captures the meaning of each skill and whether it can properly learn to compose the selected parameters (skills). Table~\ref{Composit} first shows the dialogue history along with the response of \textit{AoP} on the top, and then different responses generated by modifying $\alpha$ (i.e., black cells correspond to 1 in the vector, while the whites are 0). By analyzing Table~\ref{Composit}~\footnoteref{note1} we can notice that:

\begin{itemize}[leftmargin=*]
    \item The model learns the correct semantics of each skill. For instance, the \textit{AoP} response is of type SQL and Train, and by deactivating the SQL skill and activating other domain-skills, including Train, we can see that the responses are grammatical and they are coherent with the selected skill semantics. For instance, by just selecting \textit{Train}, the generated answer becomes \textit{``what time would you like to leave?''} which is coherent with the dialogue context since such information has not been yet provided. Interestingly, when \textit{Persona} skill is selected, the generated response is conversational and also coherent with the dialogue, even though it is less fluent.
    
    \item The model effectively learns how to compose multiple skills. For instance, when SQL or BOOK are triggered the response produces the correct SQL-syntax (e.g. ``SELECT * FROM ..'' etc.). By also adding the corresponding domain-skill, the model generates the correct query format and attributes relative to the domain type (e.g. in \textit{SQL, Restaurant}, the model queries with the relevant attribute \textit{food} for restaurants). 
\end{itemize}

\section{Related Work}
\paragraph{Dialogue} Task-oriented dialogue models~\citep{gao2018neural} can be categorized in two types: module-based~\citep{williams2007partially,hori2009statistical,lee2009example, levin2000stochastic,young2013pomdp,wu2019transferable} and end-to-end. In this paper, we focus on the latter which are systems that train a single model directly on text transcripts of dialogues. These tasks are tackled by selecting a set of predefined utterances~\citep{bordes2016learning,perez2016gated,williams2017hybrid,seo2016query} or by generating a sequence of tokens~\citep{wen2016network,serban2016building,zhao2017generative,serban2017hierarchical}. Especially in the latter, copy-augmented models~\citep{eric-manning:2017:EACLshort,reddy2018multi,yavuzdeepcopy} are very effective since extracting entities from a knowledge base is fundamental. On the other hand, end-to-end open domain chit-chat models have been widely studied~\citep{serban2016generative,vinyals2015neural,wolf2019transfertransfo,lin2019moel,lin2019caire}.
Several works improved on the initially reported baselines with various methodologies~\citep{kulikov2018importance,yavuzdeepcopy,hancock2019learning, zemlyanskiy2018aiming, dinan2019second}. Finally,~\citep{zhao2017generative} was the first attempt of having an end-to-end system for both task-oriented models and chit-chat. However, the dataset used for the evaluation was small, evaluated only in single domain, and the chit-chat ability was added manually through rules. 

\begin{table}[t]
\caption{Selecting different skills thought the attention vector $\alpha$ results in a skill-consistent response. \textit{AoP} response activates SQL and Train.}\label{Composit}
\centering
\resizebox{\textwidth}{!}{%
\begin{tabular}{llllllllll}
\hline
\multicolumn{10}{l}{\begin{tabular}[l]{@{}l@{}}\textbf{Sys}: There are lots of trains to choose from! Where are you departing from?\\ \textbf{Usr}: I am departing from london heading to cambridge.\\ \textbf{Sys}: What time will you be travelling?\\ \textbf{Usr}: I need to arrive by 1530.\\ \textbf{AoP}: SELECT * FROM train WHERE destination=``cambridge'' AND\\ \qquad \ \ day=``monday'' AND arriveBy < ``1530'' and departure=``london''\end{tabular}} \\ \hline
\multicolumn{1}{|c|}{\rotatebox{90}{SQL}} & \multicolumn{1}{c|}{\rotatebox{90}{BOOK}} & \multicolumn{1}{c|}{\rotatebox{90}{Taxi}} & \multicolumn{1}{c|}{\rotatebox{90}{Restaurant}} & \multicolumn{1}{c|}{\rotatebox{90}{Hotel}} & \multicolumn{1}{c|}{\rotatebox{90}{Attraction}} & \multicolumn{1}{c|}{\rotatebox{90}{Train}} & \multicolumn{1}{c|}{\rotatebox{90}{Schedule}} & \multicolumn{1}{c|}{\rotatebox{90}{Persona}} & \multicolumn{1}{c|}{\textit{AoP} responses using different parameter combination} \\ \hline\hline
\multicolumn{1}{|l|}{} & \multicolumn{1}{l|}{} & \multicolumn{1}{l|}{\cellcolor[HTML]{000000}} & \multicolumn{1}{l|}{} & \multicolumn{1}{l|}{} & \multicolumn{1}{l|}{} & \multicolumn{1}{l|}{} & \multicolumn{1}{l|}{} & \multicolumn{1}{l|}{} & \multicolumn{1}{l|}{I would be happy to book that for you.} \\ \hline\hline
\multicolumn{1}{|l|}{} & \multicolumn{1}{l|}{} & \multicolumn{1}{l|}{} & \multicolumn{1}{l|}{\cellcolor[HTML]{000000}} & \multicolumn{1}{l|}{} & \multicolumn{1}{l|}{} & \multicolumn{1}{l|}{} & \multicolumn{1}{l|}{} & \multicolumn{1}{l|}{} & \multicolumn{1}{l|}{I have found you a few to choose from. what time would ..} \\ \hline\hline
\multicolumn{1}{|l|}{} & \multicolumn{1}{l|}{} & \multicolumn{1}{l|}{} & \multicolumn{1}{l|}{} & \multicolumn{1}{l|}{\cellcolor[HTML]{000000}} & \multicolumn{1}{l|}{} & \multicolumn{1}{l|}{} & \multicolumn{1}{l|}{} & \multicolumn{1}{l|}{} & \multicolumn{1}{l|}{I have many options for you. is there a certain time you ...} \\ \hline\hline
\multicolumn{1}{|l|}{} & \multicolumn{1}{l|}{} & \multicolumn{1}{l|}{} & \multicolumn{1}{l|}{} & \multicolumn{1}{l|}{} & \multicolumn{1}{l|}{\cellcolor[HTML]{000000}} & \multicolumn{1}{l|}{} & \multicolumn{1}{l|}{} & \multicolumn{1}{l|}{} & \multicolumn{1}{l|}{Is there a certain area you would like?} \\ \hline\hline
\multicolumn{1}{|l|}{} & \multicolumn{1}{l|}{} & \multicolumn{1}{l|}{} & \multicolumn{1}{l|}{} & \multicolumn{1}{l|}{} & \multicolumn{1}{l|}{} & \multicolumn{1}{l|}{\cellcolor[HTML]{000000}} & \multicolumn{1}{l|}{} & \multicolumn{1}{l|}{} & \multicolumn{1}{l|}{what time would you like to leave?} \\ \hline\hline
\multicolumn{1}{|l|}{} & \multicolumn{1}{l|}{} & \multicolumn{1}{l|}{} & \multicolumn{1}{l|}{} & \multicolumn{1}{l|}{} & \multicolumn{1}{l|}{} & \multicolumn{1}{l|}{} & \multicolumn{1}{l|}{\cellcolor[HTML]{000000}} & \multicolumn{1}{l|}{} & \multicolumn{1}{l|}{okay I have two trains what time would you like to cambridge?} \\ \hline\hline
\multicolumn{1}{|l|}{} & \multicolumn{1}{l|}{} & \multicolumn{1}{l|}{} & \multicolumn{1}{l|}{} & \multicolumn{1}{l|}{} & \multicolumn{1}{l|}{} & \multicolumn{1}{l|}{} & \multicolumn{1}{l|}{} & \multicolumn{1}{l|}{\cellcolor[HTML]{000000}{\color[HTML]{000000} }} & \multicolumn{1}{l|}{Where do you do for work ?} \\ \hline\hline
\multicolumn{1}{|l|}{} & \multicolumn{1}{l|}{\cellcolor[HTML]{000000}} & \multicolumn{1}{l|}{\cellcolor[HTML]{000000}} & \multicolumn{1}{l|}{} & \multicolumn{1}{l|}{} & \multicolumn{1}{l|}{} & \multicolumn{1}{l|}{} & \multicolumn{1}{l|}{} & \multicolumn{1}{l|}{} & \multicolumn{1}{l|}{BOOK FROM taxi WHERE leaveAt>``1530'' AND destination=..} \\ \hline\hline
\multicolumn{1}{|l|}{} & \multicolumn{1}{l|}{\cellcolor[HTML]{000000}} & \multicolumn{1}{l|}{} & \multicolumn{1}{l|}{\cellcolor[HTML]{000000}} & \multicolumn{1}{l|}{} & \multicolumn{1}{l|}{} & \multicolumn{1}{l|}{} & \multicolumn{1}{l|}{} & \multicolumn{1}{l|}{} & \multicolumn{1}{l|}{BOOK FROM restaurant WHERE time=``1530'' AND ...} \\ \hline\hline
\multicolumn{1}{|l|}{} & \multicolumn{1}{l|}{\cellcolor[HTML]{000000}} & \multicolumn{1}{l|}{} & \multicolumn{1}{l|}{} & \multicolumn{1}{l|}{\cellcolor[HTML]{000000}} & \multicolumn{1}{l|}{} & \multicolumn{1}{l|}{} & \multicolumn{1}{l|}{} & \multicolumn{1}{l|}{} & \multicolumn{1}{l|}{BOOK FROM hotel WHERE people=``1'' AND day=``monday''} \\ \hline\hline
\multicolumn{1}{|l|}{} & \multicolumn{1}{l|}{\cellcolor[HTML]{000000}} & \multicolumn{1}{l|}{} & \multicolumn{1}{l|}{} & \multicolumn{1}{l|}{} & \multicolumn{1}{l|}{} & \multicolumn{1}{l|}{\cellcolor[HTML]{000000}} & \multicolumn{1}{l|}{} & \multicolumn{1}{l|}{} & \multicolumn{1}{l|}{BOOK FROM train WHERE people=``1'' AND id\_booking='} \\ \hline\hline
\multicolumn{1}{|l|}{\cellcolor[HTML]{000000}} & \multicolumn{1}{l|}{} & \multicolumn{1}{l|}{} & \multicolumn{1}{l|}{\cellcolor[HTML]{000000}} & \multicolumn{1}{l|}{} & \multicolumn{1}{l|}{} & \multicolumn{1}{l|}{} & \multicolumn{1}{l|}{} & \multicolumn{1}{l|}{} & \multicolumn{1}{l|}{SELECT * FROM restaurant WHERE food=``1530'' ...} \\ \hline\hline
\multicolumn{1}{|l|}{\cellcolor[HTML]{000000}} & \multicolumn{1}{l|}{} & \multicolumn{1}{l|}{} & \multicolumn{1}{l|}{} & \multicolumn{1}{l|}{\cellcolor[HTML]{000000}} & \multicolumn{1}{l|}{} & \multicolumn{1}{l|}{} & \multicolumn{1}{l|}{} & \multicolumn{1}{l|}{} & \multicolumn{1}{l|}{SELECT * FROM hotel WHERE type=``1530''} \\ \hline\hline
\multicolumn{1}{|l|}{\cellcolor[HTML]{000000}} & \multicolumn{1}{l|}{} & \multicolumn{1}{l|}{} & \multicolumn{1}{l|}{} & \multicolumn{1}{l|}{} & \multicolumn{1}{l|}{\cellcolor[HTML]{000000}} & \multicolumn{1}{l|}{} & \multicolumn{1}{l|}{} & \multicolumn{1}{l|}{} & \multicolumn{1}{l|}{SELECT * FROM attraction WHERE name=``departure''} \\ \hline \hline
\end{tabular}}
\end{table}

\paragraph{Mixture of Expert \& Conditional Computation}
The idea of having specialized parameters, or so-called experts, has been widely studied topics in the last two decades~\citep{jacobs1991adaptive,jordan1994hierarchical}. For instance, different architecture and methodologies have been used such as Gaussian Processes~\citep{tresp2001mixtures}, Hierarchical Experts~\citep{yao2009hierarchical}, and sequential expert addition~\citep{aljundi2017expert}. More recently, the Mixture Of Expert~\citep{shazeer2017outrageously,kaiser2017one} model was proposed which added a large number of experts between two LSTMs. To the best of our knowledge, none of these previous works applied the results of the gating function to the parameters itself. On the other hand, there are Conditional Computational models which learn to dynamically select their computation graph~\citep{bengio2013estimating,davis2013low}. Several methods have been used such as reinforcement learning~\citep{bengio2016conditional}, a halting function~\citep{graves2016adaptive,dehghani2018universal,figurnov2017spatially}, by pruning~\citep{lin2017runtime,he2018amc} and routing/controller function~\citep{rosenbaum2018routing}. However, this line of work focuses more on optimizing the inference performance of the model more than specializing parts of it for computing a certain task.

\paragraph{Multi-task Learning}
Even though our model processes only input sequence and output sequences of text, it actually jointly learns multiple tasks (e.g. SQL and BOOK query, memory retrieval, and response generation), thus it is also related to multi-task learning~\citep{caruana1997multitask}. Interested readers may refer to~\citep{ruder2017overview,zhou2011clustered} for a general overview on the topic. In Natural Language Processing, multi-task learning has been applied in a wide range of applications such as parsing~\citep{collobert2011natural,hashimoto2017joint,ruder2017learning}, machine translation in multiple languages~\citep{johnson2017google}, and parsing image captioning and machine translation~\citep{44928}. More interestingly, DecaNLP~\citep{mccann2018natural} has a large set of tasks that are cast to question answering (QA), and learned by a single model. In this work, we focus more on conversational data, but in future works, we plan to include these QA tasks.

\section{Conclusion}
In this paper, we propose a novel way to train a single end-to-end dialogue model with multiple composable and interpretable skills. Unlike previous work, that mostly focused on the representation-level mixing~\citep{shazeer2017outrageously}, our proposed approach, \textit{Attention over Parameters}, learns how to softly combine independent sets of specialized parameters (i.e., making SQL-Query, conversing with consistent persona, etc.) into a single set of parameters. By doing so, we not only achieve compositionality and interpretability but also gain algorithmically faster inference speed. To train and evaluate our model, we organize a multi-domain task-oriented datasets into end-to-end trainable formats and combine it with a conversational dataset (i.e. Persona-Chat). Our model learns to consider each task and domain as a separate skill that can be composed with each other, or used independently, and we verify the effectiveness of the interpretability and compositionality with competitive experimental results and thorough analysis.

Several extensions of this work are possible, for example: incremental learning and zero-shot skill composition. The first, would be similar to~\cite{rusu2016progressive} where we can add skills through time and learn how to combine it to existing ones. The second, instead, is more related to the semantic compositionality shown in the analysis, where each skill is correctly learned and can be apply to control the generation. An interesting direction would be to learn more general skills (e.g. Machine Translation (MT) or emotional responses), and being able to mix it to existing skills to obtain compositional responses without labeled data.

\bibliographystyle{unsrtnat}
\bibliography{neurips_2019}

\begin{thebibliography}{73}
\providecommand{\natexlab}[1]{#1}
\providecommand{\url}[1]{\texttt{#1}}
\expandafter\ifx\csname urlstyle\endcsname\relax
  \providecommand{\doi}[1]{doi: #1}\else
  \providecommand{\doi}{doi: \begingroup \urlstyle{rm}\Url}\fi

\bibitem[Budzianowski et~al.(2018)Budzianowski, Wen, Tseng, Casanueva, Ultes,
  Ramadan, and Gasic]{budzianowski2018multiwoz}
Pawe{\l} Budzianowski, Tsung-Hsien Wen, Bo-Hsiang Tseng, I{\~n}igo Casanueva,
  Stefan Ultes, Osman Ramadan, and Milica Gasic.
\newblock Multiwoz-a large-scale multi-domain wizard-of-oz dataset for
  task-oriented dialogue modelling.
\newblock In \emph{Proceedings of the 2018 Conference on Empirical Methods in
  Natural Language Processing}, pages 5016--5026, 2018.

\bibitem[Eric et~al.(2017)Eric, Krishnan, Charette, and Manning]{ericKVR2017}
Mihail Eric, Lakshmi Krishnan, Francois Charette, and Christopher~D. Manning.
\newblock Key-value retrieval networks for task-oriented dialogue.
\newblock In \emph{Proceedings of the 18th Annual SIGdial Meeting on Discourse
  and Dialogue}, pages 37--49. Association for Computational Linguistics, 2017.
\newblock URL \url{http://aclweb.org/anthology/W17-5506}.

\bibitem[Zhang et~al.(2018)Zhang, Dinan, Urbanek, Szlam, Kiela, and
  Weston]{personachat}
Saizheng Zhang, Emily Dinan, Jack Urbanek, Arthur Szlam, Douwe Kiela, and Jason
  Weston.
\newblock Personalizing dialogue agents: I have a dog, do you have pets too?
\newblock In \emph{Proceedings of the 56th Annual Meeting of the Association
  for Computational Linguistics (Volume 1: Long Papers)}, pages 2204--2213.
  Association for Computational Linguistics, 2018.
\newblock URL \url{http://aclweb.org/anthology/P18-1205}.

\bibitem[Williams and Young(2007)]{williams2007partially}
Jason~D Williams and Steve Young.
\newblock Partially observable markov decision processes for spoken dialog
  systems.
\newblock \emph{Computer Speech \& Language}, 21\penalty0 (2):\penalty0
  393--422, 2007.

\bibitem[Young et~al.(2013)Young, Ga{\v{s}}i{\'c}, Thomson, and
  Williams]{young2013pomdp}
Steve Young, Milica Ga{\v{s}}i{\'c}, Blaise Thomson, and Jason~D Williams.
\newblock Pomdp-based statistical spoken dialog systems: A review.
\newblock \emph{Proceedings of the IEEE}, 101\penalty0 (5):\penalty0
  1160--1179, 2013.

\bibitem[Serban et~al.(2016{\natexlab{a}})Serban, Lowe, Charlin, and
  Pineau]{serban2016generative}
Iulian~Vlad Serban, Ryan Lowe, Laurent Charlin, and Joelle Pineau.
\newblock Generative deep neural networks for dialogue: A short review.
\newblock \emph{arXiv preprint arXiv:1611.06216}, 2016{\natexlab{a}}.

\bibitem[Vinyals and Le(2015)]{vinyals2015neural}
Oriol Vinyals and Quoc~V Le.
\newblock A neural conversational model.
\newblock \emph{arXiv preprint arXiv:1506.05869}, 2015.

\bibitem[Hori et~al.(2009)Hori, Ohtake, Misu, Kashioka, and
  Nakamura]{hori2009statistical}
Chiori Hori, Kiyonori Ohtake, Teruhisa Misu, Hideki Kashioka, and Satoshi
  Nakamura.
\newblock Statistical dialog management applied to wfst-based dialog systems.
\newblock In \emph{IEEE International Conference on Acoustics, Speech and
  Signal Processing, 2009. ICASSP 2009.}, pages 4793--4796. IEEE, 2009.

\bibitem[Lee et~al.(2009)Lee, Jung, Kim, and Lee]{lee2009example}
Cheongjae Lee, Sangkeun Jung, Seokhwan Kim, and Gary~Geunbae Lee.
\newblock Example-based dialog modeling for practical multi-domain dialog
  system.
\newblock \emph{Speech Communication}, 51\penalty0 (5):\penalty0 466--484,
  2009.

\bibitem[Levin et~al.(2000)Levin, Pieraccini, and Eckert]{levin2000stochastic}
Esther Levin, Roberto Pieraccini, and Wieland Eckert.
\newblock A stochastic model of human-machine interaction for learning dialog
  strategies.
\newblock \emph{IEEE Transactions on speech and audio processing}, 8\penalty0
  (1):\penalty0 11--23, 2000.

\bibitem[Eric and Manning(2017)]{eric-manning:2017:EACLshort}
Mihail Eric and Christopher Manning.
\newblock A copy-augmented sequence-to-sequence architecture gives good
  performance on task-oriented dialogue.
\newblock In \emph{Proceedings of the 15th Conference of the European Chapter
  of the Association for Computational Linguistics: Volume 2, Short Papers},
  pages 468--473, Valencia, Spain, April 2017. Association for Computational
  Linguistics.
\newblock URL \url{http://www.aclweb.org/anthology/E17-2075}.

\bibitem[Wu et~al.(2019{\natexlab{a}})Wu, Socher, and Xiong]{wu2019global}
Chien-Sheng Wu, Richard Socher, and Caiming Xiong.
\newblock Global-to-local memory pointer networks for task-oriented dialogue.
\newblock In \emph{Proceedings of the International Conference on Learning
  Representations (ICLR)}, 2019{\natexlab{a}}.

\bibitem[Reddy et~al.(2018)Reddy, Contractor, Raghu, and Joshi]{reddy2018multi}
Revanth Reddy, Danish Contractor, Dinesh Raghu, and Sachindra Joshi.
\newblock Multi-level memory for task oriented dialogs.
\newblock \emph{arXiv preprint arXiv:1810.10647}, 2018.

\bibitem[Yavuz et~al.(2018)Yavuz, Rastogi, Chao, Hakkani-T{\"u}r, and
  AI]{yavuzdeepcopy}
Semih Yavuz, Abhinav Rastogi, Guan-lin Chao, Dilek Hakkani-T{\"u}r, and
  Amazon~Alexa AI.
\newblock Deepcopy: Grounded response generation with hierarchical pointer
  networks.
\newblock \emph{Conversational AI NIPS workshop}, 2018.

\bibitem[Bordes and Weston(2017)]{bordes2016learning}
Antoine Bordes and Jason Weston.
\newblock Learning end-to-end goal-oriented dialog.
\newblock \emph{International Conference on Learning Representations},
  abs/1605.07683, 2017.

\bibitem[Zhao et~al.(2017)Zhao, Lu, Lee, and Eskenazi]{zhao2017generative}
Tiancheng Zhao, Allen Lu, Kyusong Lee, and Maxine Eskenazi.
\newblock Generative encoder-decoder models for task-oriented spoken dialog
  systems with chatting capability.
\newblock In \emph{Proceedings of the 18th Annual SIGdial Meeting on Discourse
  and Dialogue}, pages 27--36. Association for Computational Linguistics,
  August 2017.
\newblock URL \url{http://aclweb.org/anthology/W17-5505}.

\bibitem[McCann et~al.(2018)McCann, Keskar, Xiong, and
  Socher]{mccann2018natural}
Bryan McCann, Nitish~Shirish Keskar, Caiming Xiong, and Richard Socher.
\newblock The natural language decathlon: Multitask learning as question
  answering.
\newblock \emph{arXiv preprint arXiv:1806.08730}, 2018.

\bibitem[Jacobs et~al.(1991)Jacobs, Jordan, Nowlan, Hinton,
  et~al.]{jacobs1991adaptive}
Robert~A Jacobs, Michael~I Jordan, Steven~J Nowlan, Geoffrey~E Hinton, et~al.
\newblock Adaptive mixtures of local experts.
\newblock \emph{Neural computation}, 3\penalty0 (1):\penalty0 79--87, 1991.

\bibitem[Shazeer et~al.(2017)Shazeer, Mirhoseini, Maziarz, Davis, Le, Hinton,
  and Dean]{shazeer2017outrageously}
Noam Shazeer, Azalia Mirhoseini, Krzysztof Maziarz, Andy Davis, Quoc Le,
  Geoffrey Hinton, and Jeff Dean.
\newblock Outrageously large neural networks: The sparsely-gated
  mixture-of-experts layer.
\newblock \emph{arXiv preprint arXiv:1701.06538}, 2017.

\bibitem[Kaiser et~al.(2017)Kaiser, Gomez, Shazeer, Vaswani, Parmar, Jones, and
  Uszkoreit]{kaiser2017one}
Lukasz Kaiser, Aidan~N Gomez, Noam Shazeer, Ashish Vaswani, Niki Parmar, Llion
  Jones, and Jakob Uszkoreit.
\newblock One model to learn them all.
\newblock \emph{arXiv}, 2017.

\bibitem[Schmidhuber(1987)]{schmidhuber:1987:srl}
Jurgen Schmidhuber.
\newblock Evolutionary principles in self-referential learning. on learning now
  to learn: The meta-meta-meta...-hook.
\newblock Diploma thesis, Technische Universitat Munchen, Germany, 14 May 1987.
\newblock URL \url{http://www.idsia.ch/~juergen/diploma.html}.

\bibitem[Vaswani et~al.(2017)Vaswani, Shazeer, Parmar, Uszkoreit, Jones, Gomez,
  Kaiser, and Polosukhin]{vaswani2017attention}
Ashish Vaswani, Noam Shazeer, Niki Parmar, Jakob Uszkoreit, Llion Jones,
  Aidan~N Gomez, {\L}ukasz Kaiser, and Illia Polosukhin.
\newblock Attention is all you need.
\newblock In \emph{Advances in Neural Information Processing Systems}, pages
  6000--6010, 2017.

\bibitem[Williams and Zipser(1989)]{williams1989learning}
Ronald~J Williams and David Zipser.
\newblock A learning algorithm for continually running fully recurrent neural
  networks.
\newblock \emph{Neural computation}, 1\penalty0 (2):\penalty0 270--280, 1989.

\bibitem[See et~al.(2017)See, Liu, and Manning]{see-liu-manning:2017:Long}
Abigail See, Peter~J. Liu, and Christopher~D. Manning.
\newblock Get to the point: Summarization with pointer-generator networks.
\newblock In \emph{Proceedings of the 55th Annual Meeting of the Association
  for Computational Linguistics (Volume 1: Long Papers)}, pages 1073--1083,
  Vancouver, Canada, July 2017. Association for Computational Linguistics.
\newblock URL \url{http://aclweb.org/anthology/P17-1099}.

\bibitem[Graves(2013)]{graves2013generating}
Alex Graves.
\newblock Generating sequences with recurrent neural networks.
\newblock \emph{arXiv preprint arXiv:1308.0850}, 2013.

\bibitem[Luong et~al.(2015)Luong, Pham, and
  Manning]{luong-pham-manning:2015:EMNLP}
Thang Luong, Hieu Pham, and Christopher~D. Manning.
\newblock Effective approaches to attention-based neural machine translation.
\newblock In \emph{Proceedings of the 2015 Conference on Empirical Methods in
  Natural Language Processing}, pages 1412--1421, Lisbon, Portugal, September
  2015. Association for Computational Linguistics.
\newblock URL \url{http://aclweb.org/anthology/D15-1166}.

\bibitem[Miller et~al.(2016)Miller, Fisch, Dodge, Karimi, Bordes, and
  Weston]{miller2016key}
Alexander Miller, Adam Fisch, Jesse Dodge, Amir-Hossein Karimi, Antoine Bordes,
  and Jason Weston.
\newblock Key-value memory networks for directly reading documents.
\newblock In \emph{Proceedings of the 2016 Conference on Empirical Methods in
  Natural Language Processing}, pages 1400--1409, 2016.

\bibitem[Cho et~al.(2014)Cho, van Merrienboer, Gulcehre, Bahdanau, Bougares,
  Schwenk, and Bengio]{cho2014learning}
Kyunghyun Cho, Bart van Merrienboer, Caglar Gulcehre, Dzmitry Bahdanau, Fethi
  Bougares, Holger Schwenk, and Yoshua Bengio.
\newblock Learning phrase representations using rnn encoder--decoder for
  statistical machine translation.
\newblock In \emph{Proceedings of the 2014 Conference on Empirical Methods in
  Natural Language Processing (EMNLP)}, pages 1724--1734, 2014.

\bibitem[Dehghani et~al.(2019)Dehghani, Gouws, Vinyals, Uszkoreit, and
  Kaiser]{dehghani2018universal}
Mostafa Dehghani, Stephan Gouws, Oriol Vinyals, Jakob Uszkoreit, and Lukasz
  Kaiser.
\newblock Universal transformers.
\newblock \emph{ICLR}, 2019.

\bibitem[Madotto et~al.(2018)Madotto, Wu, and Fung]{madotto2018mem2seq}
Andrea Madotto, Chien-Sheng Wu, and Pascale Fung.
\newblock Mem2seq: Effectively incorporating knowledge bases into end-to-end
  task-oriented dialog systems.
\newblock \emph{arXiv preprint arXiv:1804.08217}, 2018.

\bibitem[Wu et~al.(2017)Wu, Madotto, Winata, and Fung]{wu2017dstc6}
Chien-Sheng Wu, Andrea Madotto, Genta Winata, and Pascale Fung.
\newblock End-to-end recurrent entity network for entity-value independent
  goal-oriented dialog learning.
\newblock In \emph{Dialog System Technology Challenges Workshop, DSTC6}, 2017.

\bibitem[Papineni et~al.(2002)Papineni, Roukos, Ward, and
  Zhu]{pAPIneniBLEU2002}
Kishore Papineni, Salim Roukos, Todd Ward, and Wei-Jing Zhu.
\newblock Bleu: a method for automatic evaluation of machine translation.
\newblock In \emph{Proceedings of 40th Annual Meeting of the Association for
  Computational Linguistics}, pages 311--318, Philadelphia, Pennsylvania, USA,
  July 2002. Association for Computational Linguistics.
\newblock \doi{10.3115/1073083.1073135}.
\newblock URL \url{http://www.aclweb.org/anthology/P02-1040}.

\bibitem[Wen et~al.(2017)Wen, Gasic, Mrksic, Rojas-Barahona, hao Su, Ultes,
  Vandyke, and Young]{wen2016network}
Tsung-Hsien Wen, Milica Gasic, Nikola Mrksic, Lina~Maria Rojas-Barahona, Pei
  hao Su, Stefan Ultes, David Vandyke, and Steve~J. Young.
\newblock A network-based end-to-end trainable task-oriented dialogue system.
\newblock In \emph{EACL}, 2017.

\bibitem[Dinan et~al.(2019)Dinan, Logacheva, Malykh, Miller, Shuster, Urbanek,
  Kiela, Szlam, Serban, Lowe, et~al.]{dinan2019second}
Emily Dinan, Varvara Logacheva, Valentin Malykh, Alexander Miller, Kurt
  Shuster, Jack Urbanek, Douwe Kiela, Arthur Szlam, Iulian Serban, Ryan Lowe,
  et~al.
\newblock The second conversational intelligence challenge (convai2).
\newblock \emph{arXiv preprint arXiv:1902.00098}, 2019.

\bibitem[Sean et~al.(2018)Sean, Weston, Szlam, and Cho]{dnli}
Welleck Sean, Jason Weston, Arthur Szlam, and Kyunghyun Cho.
\newblock Dialogue natural language inference.
\newblock \emph{arXiv preprint arXiv:1811.00671}, 2018.

\bibitem[Devlin et~al.(2018)Devlin, Chang, Lee, and Toutanova]{devlin2018bert}
Jacob Devlin, Ming-Wei Chang, Kenton Lee, and Kristina Toutanova.
\newblock Bert: Pre-training of deep bidirectional transformers for language
  understanding.
\newblock \emph{arXiv preprint arXiv:1810.04805}, 2018.

\bibitem[Madotto et~al.(2019)Madotto, Lin, Wu, and
  Fung]{madotto2019personalizing}
Andrea Madotto, Zhaojiang Lin, Chien-Sheng Wu, and Pascale Fung.
\newblock Personalizing dialogue agents via meta-learning.
\newblock In \emph{Proceedings of the 57th Annual Meeting of the Association
  for Computational Linguistics}, pages 5454--5459, 2019.

\bibitem[Gao et~al.(2018)Gao, Galley, and Li]{gao2018neural}
Jianfeng Gao, Michel Galley, and Lihong Li.
\newblock Neural approaches to conversational ai.
\newblock In \emph{The 41st International ACM SIGIR Conference on Research \&
  Development in Information Retrieval}, pages 1371--1374. ACM, 2018.

\bibitem[Wu et~al.(2019{\natexlab{b}})Wu, Madotto, Hosseini-Asl, Xiong, Socher,
  and Fung]{wu2019transferable}
Chien-Sheng Wu, Andrea Madotto, Ehsan Hosseini-Asl, Caiming Xiong, Richard
  Socher, and Pascale Fung.
\newblock Transferable multi-domain state generator for task-oriented dialogue
  systems.
\newblock \emph{arXiv preprint arXiv:1905.08743}, 2019{\natexlab{b}}.

\bibitem[Liu and Perez(2017)]{perez2016gated}
Fei Liu and Julien Perez.
\newblock Gated end-to-end memory networks.
\newblock In \emph{Proceedings of the 15th Conference of the European Chapter
  of the Association for Computational Linguistics: Volume 1, Long Papers},
  pages 1--10, Valencia, Spain, April 2017. Association for Computational
  Linguistics.
\newblock URL \url{http://www.aclweb.org/anthology/E17-1001}.

\bibitem[Williams et~al.(2017)Williams, Asadi, and Zweig]{williams2017hybrid}
Jason~D Williams, Kavosh Asadi, and Geoffrey Zweig.
\newblock Hybrid code networks: practical and efficient end-to-end dialog
  control with supervised and reinforcement learning.
\newblock In \emph{Proceedings of the 55th Annual Meeting of the Association
  for Computational Linguistics (Volume 1: Long Papers)}, pages 665--677,
  Vancouver, Canada, July 2017. Association for Computational Linguistics.
\newblock URL \url{http://aclweb.org/anthology/P17-1062}.

\bibitem[Seo et~al.(2017)Seo, Min, Farhadi, and Hajishirzi]{seo2016query}
Minjoon Seo, Sewon Min, Ali Farhadi, and Hannaneh Hajishirzi.
\newblock Query-reduction networks for question answering.
\newblock \emph{International Conference on Learning Representations}, 2017.

\bibitem[Serban et~al.(2016{\natexlab{b}})Serban, Sordoni, Bengio, Courville,
  and Pineau]{serban2016building}
Iulian~Vlad Serban, Alessandro Sordoni, Yoshua Bengio, Aaron~C Courville, and
  Joelle Pineau.
\newblock Building end-to-end dialogue systems using generative hierarchical
  neural network models.
\newblock In \emph{AAAI}, pages 3776--3784, 2016{\natexlab{b}}.

\bibitem[Serban et~al.(2017)Serban, Sordoni, Lowe, Charlin, Pineau, Courville,
  and Bengio]{serban2017hierarchical}
Iulian~Vlad Serban, Alessandro Sordoni, Ryan Lowe, Laurent Charlin, Joelle
  Pineau, Aaron~C Courville, and Yoshua Bengio.
\newblock A hierarchical latent variable encoder-decoder model for generating
  dialogues.
\newblock In \emph{AAAI}, pages 3295--3301, 2017.

\bibitem[Wolf et~al.(2019)Wolf, Sanh, Chaumond, and
  Delangue]{wolf2019transfertransfo}
Thomas Wolf, Victor Sanh, Julien Chaumond, and Clement Delangue.
\newblock Transfertransfo: A transfer learning approach for neural network
  based conversational agents.
\newblock \emph{arXiv preprint arXiv:1901.08149}, 2019.

\bibitem[Lin et~al.(2019{\natexlab{a}})Lin, Madotto, Shin, Xu, and
  Fung]{lin2019moel}
Zhaojiang Lin, Andrea Madotto, Jamin Shin, Peng Xu, and Pascale Fung.
\newblock Moel: Mixture of empathetic listeners.
\newblock In \emph{Proceedings of the 2019 Conference on Empirical Methods in
  Natural Language Processing and the 9th International Joint Conference on
  Natural Language Processing (EMNLP-IJCNLP)}, pages 121--132,
  2019{\natexlab{a}}.

\bibitem[Lin et~al.(2019{\natexlab{b}})Lin, Xu, Winata, Liu, and
  Fung]{lin2019caire}
Zhaojiang Lin, Peng Xu, Genta~Indra Winata, Zihan Liu, and Pascale Fung.
\newblock Caire: An end-to-end empathetic chatbot.
\newblock \emph{arXiv preprint arXiv:1907.12108}, 2019{\natexlab{b}}.

\bibitem[Kulikov et~al.(2018)Kulikov, Miller, Cho, and
  Weston]{kulikov2018importance}
Ilya Kulikov, Alexander~H Miller, Kyunghyun Cho, and Jason Weston.
\newblock Importance of a search strategy in neural dialogue modelling.
\newblock \emph{arXiv preprint arXiv:1811.00907}, 2018.

\bibitem[Hancock et~al.(2019)Hancock, Bordes, Mazare, and
  Weston]{hancock2019learning}
Braden Hancock, Antoine Bordes, Pierre-Emmanuel Mazare, and Jason Weston.
\newblock Learning from dialogue after deployment: Feed yourself, chatbot!
\newblock \emph{arXiv preprint arXiv:1901.05415}, 2019.

\bibitem[Zemlyanskiy and Sha(2018)]{zemlyanskiy2018aiming}
Yury Zemlyanskiy and Fei Sha.
\newblock Aiming to know you better perhaps makes me a more engaging dialogue
  partner.
\newblock \emph{CoNLL 2018}, page 551, 2018.

\bibitem[Jordan and Jacobs(1994)]{jordan1994hierarchical}
Michael~I Jordan and Robert~A Jacobs.
\newblock Hierarchical mixtures of experts and the em algorithm.
\newblock \emph{Neural computation}, 6\penalty0 (2):\penalty0 181--214, 1994.

\bibitem[Tresp(2001)]{tresp2001mixtures}
Volker Tresp.
\newblock Mixtures of gaussian processes.
\newblock In \emph{Advances in neural information processing systems}, pages
  654--660, 2001.

\bibitem[Yao et~al.(2009)Yao, Walther, Beck, and Fei-Fei]{yao2009hierarchical}
Bangpeng Yao, Dirk Walther, Diane Beck, and Li~Fei-Fei.
\newblock Hierarchical mixture of classification experts uncovers interactions
  between brain regions.
\newblock In \emph{Advances in Neural Information Processing Systems}, pages
  2178--2186, 2009.

\bibitem[Aljundi et~al.(2017)Aljundi, Chakravarty, and
  Tuytelaars]{aljundi2017expert}
Rahaf Aljundi, Punarjay Chakravarty, and Tinne Tuytelaars.
\newblock Expert gate: Lifelong learning with a network of experts.
\newblock In \emph{Proceedings of the IEEE Conference on Computer Vision and
  Pattern Recognition}, pages 3366--3375, 2017.

\bibitem[Bengio et~al.(2013)Bengio, L{\'e}onard, and
  Courville]{bengio2013estimating}
Yoshua Bengio, Nicholas L{\'e}onard, and Aaron Courville.
\newblock Estimating or propagating gradients through stochastic neurons for
  conditional computation.
\newblock \emph{arXiv preprint arXiv:1308.3432}, 2013.

\bibitem[Davis and Arel(2013)]{davis2013low}
Andrew Davis and Itamar Arel.
\newblock Low-rank approximations for conditional feedforward computation in
  deep neural networks.
\newblock \emph{arXiv}, 2013.

\bibitem[Bengio et~al.(2016)Bengio, Bacon, Pineau, and
  Precup]{bengio2016conditional}
Emmanuel Bengio, Pierre-Luc Bacon, Joelle Pineau, and Doina Precup.
\newblock Conditional computation in neural networks for faster models.
\newblock \emph{ICLR}, 2016.

\bibitem[Graves(2016)]{graves2016adaptive}
Alex Graves.
\newblock Adaptive computation time for recurrent neural networks.
\newblock \emph{arXiv preprint arXiv:1603.08983}, 2016.

\bibitem[Figurnov et~al.(2017)Figurnov, Collins, Zhu, Zhang, Huang, Vetrov, and
  Salakhutdinov]{figurnov2017spatially}
Michael Figurnov, Maxwell~D Collins, Yukun Zhu, Li~Zhang, Jonathan Huang,
  Dmitry Vetrov, and Ruslan Salakhutdinov.
\newblock Spatially adaptive computation time for residual networks.
\newblock In \emph{Proceedings of the IEEE Conference on Computer Vision and
  Pattern Recognition}, pages 1039--1048, 2017.

\bibitem[Lin et~al.(2017)Lin, Rao, Lu, and Zhou]{lin2017runtime}
Ji~Lin, Yongming Rao, Jiwen Lu, and Jie Zhou.
\newblock Runtime neural pruning.
\newblock In \emph{Advances in Neural Information Processing Systems}, pages
  2181--2191, 2017.

\bibitem[He et~al.(2018)He, Lin, Liu, Wang, Li, and Han]{he2018amc}
Yihui He, Ji~Lin, Zhijian Liu, Hanrui Wang, Li-Jia Li, and Song Han.
\newblock Amc: Automl for model compression and acceleration on mobile devices.
\newblock In \emph{Proceedings of the European Conference on Computer Vision
  (ECCV)}, pages 784--800, 2018.

\bibitem[Rosenbaum et~al.(2018)Rosenbaum, Klinger, and
  Riemer]{rosenbaum2018routing}
Clemens Rosenbaum, Tim Klinger, and Matthew Riemer.
\newblock Routing networks: Adaptive selection of non-linear functions for
  multi-task learning.
\newblock In \emph{International Conference on Learning Representations}, 2018.
\newblock URL \url{https://openreview.net/forum?id=ry8dvM-R-}.

\bibitem[Caruana(1997)]{caruana1997multitask}
Rich Caruana.
\newblock Multitask learning.
\newblock \emph{Machine learning}, 28\penalty0 (1):\penalty0 41--75, 1997.

\bibitem[Ruder(2017)]{ruder2017overview}
Sebastian Ruder.
\newblock An overview of multi-task learning in deep neural networks.
\newblock \emph{arXiv preprint arXiv:1706.05098}, 2017.

\bibitem[Zhou et~al.(2011)Zhou, Chen, and Ye]{zhou2011clustered}
Jiayu Zhou, Jianhui Chen, and Jieping Ye.
\newblock Clustered multi-task learning via alternating structure optimization.
\newblock In \emph{Advances in neural information processing systems}, pages
  702--710, 2011.

\bibitem[Collobert et~al.(2011)Collobert, Weston, Bottou, Karlen, Kavukcuoglu,
  and Kuksa]{collobert2011natural}
Ronan Collobert, Jason Weston, L{\'e}on Bottou, Michael Karlen, Koray
  Kavukcuoglu, and Pavel Kuksa.
\newblock Natural language processing (almost) from scratch.
\newblock \emph{Journal of machine learning research}, 12\penalty0
  (Aug):\penalty0 2493--2537, 2011.

\bibitem[Hashimoto et~al.(2017)Hashimoto, Tsuruoka, Socher,
  et~al.]{hashimoto2017joint}
Kazuma Hashimoto, Yoshimasa Tsuruoka, Richard Socher, et~al.
\newblock A joint many-task model: Growing a neural network for multiple nlp
  tasks.
\newblock In \emph{Proceedings of the 2017 Conference on Empirical Methods in
  Natural Language Processing}, pages 1923--1933, 2017.

\bibitem[Ruder et~al.(2017)Ruder, Bingel, Augenstein, and
  S{\o}gaard]{ruder2017learning}
Sebastian Ruder, Joachim Bingel, Isabelle Augenstein, and Anders S{\o}gaard.
\newblock Learning what to share between loosely related tasks.
\newblock \emph{arXiv preprint arXiv:1705.08142}, 2017.

\bibitem[Johnson et~al.(2017)Johnson, Schuster, Le, Krikun, Wu, Chen, Thorat,
  Vi{\'e}gas, Wattenberg, Corrado, et~al.]{johnson2017google}
Melvin Johnson, Mike Schuster, Quoc~V Le, Maxim Krikun, Yonghui Wu, Zhifeng
  Chen, Nikhil Thorat, Fernanda Vi{\'e}gas, Martin Wattenberg, Greg Corrado,
  et~al.
\newblock Google’s multilingual neural machine translation system: Enabling
  zero-shot translation.
\newblock \emph{Transactions of the Association for Computational Linguistics},
  5:\penalty0 339--351, 2017.

\bibitem[Luong et~al.(2016)Luong, Le, Sutskever, Vinyals, and Kaiser]{44928}
Thang Luong, Quoc~V. Le, Ilya Sutskever, Oriol Vinyals, and Lukasz Kaiser.
\newblock Multi-task sequence to sequence learning.
\newblock In \emph{International Conference on Learning Representations}, 2016.

\bibitem[Rusu et~al.(2016)Rusu, Rabinowitz, Desjardins, Soyer, Kirkpatrick,
  Kavukcuoglu, Pascanu, and Hadsell]{rusu2016progressive}
Andrei~A Rusu, Neil~C Rabinowitz, Guillaume Desjardins, Hubert Soyer, James
  Kirkpatrick, Koray Kavukcuoglu, Razvan Pascanu, and Raia Hadsell.
\newblock Progressive neural networks.
\newblock \emph{arXiv preprint arXiv:1606.04671}, 2016.

\bibitem[Pennington et~al.(2014)Pennington, Socher, and
  Manning]{pennington2014glove}
Jeffrey Pennington, Richard Socher, and Christopher Manning.
\newblock Glove: Global vectors for word representation.
\newblock In \emph{Proceedings of the 2014 conference on empirical methods in
  natural language processing (EMNLP)}, pages 1532--1543, 2014.

\bibitem[Kingma and Ba(2014)]{kingma2014adam}
Diederik~P Kingma and Jimmy Ba.
\newblock Adam: A method for stochastic optimization.
\newblock \emph{arXiv preprint arXiv:1412.6980}, 2014.

\end{thebibliography}

\newpage
\appendix
\section{Appendix}

\subsection{Embedded Representation}
\begin{wrapfigure}{r}{0.5\textwidth}
    \vspace{-30pt}
  \begin{center}
    \includegraphics[width=\linewidth]{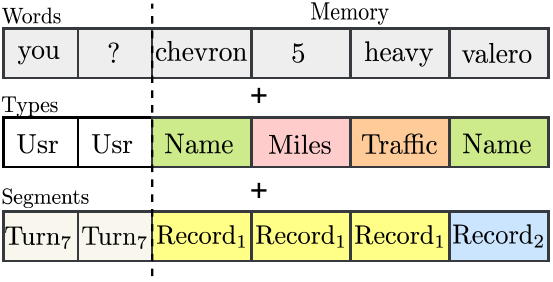}
  \end{center}
    \caption{Positional Embedding of the dialogue history and the memory content.}
    \label{example_position}
    \vspace{-10pt}
\end{wrapfigure}
Since the model input may include structured data (e.g. DB records) we further define another embedding matrix for encoding the types and the segments as $P\in \mathbb{R}^{d \times |S|}$ where $S$ is the set of positional tokens and $|S|$ its cardinality. $P$ is used to inform the model of the token types such as speaker information (e.g. \textit{Sys} and \textit{Usr}), the data-type for the memory content (e.g. \textit{Miles}, \textit{Traffic} etc.), and segment types like dialogue turn information and database record index~\citep{wolf2019transfertransfo}. Figure~\ref{example_position} shows an example of the embedded representation of the input. Hence, we denote $X_T$ and $X_R$ as the type and segment tokens for each token in input $X$, respectively.

\begin{figure}[t]
    \centering
    \includegraphics[width=0.98\linewidth]{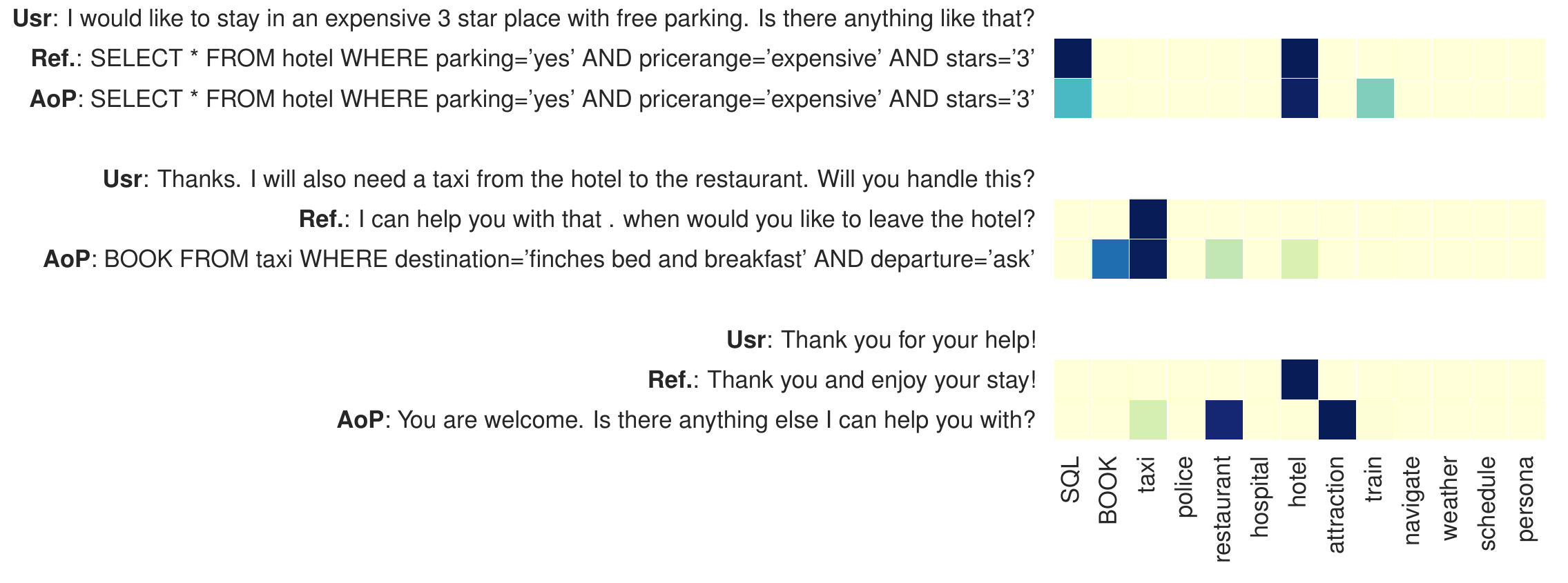}
    \caption{Attention over Parameters visualization, vector $\alpha$ for different reference (\textbf{Ref.}) and \textit{AoP} generated answers. Top rows (\textbf{Usr}) are the last utterances from each dialogue contexts.}
    \label{fig:Viz}
\end{figure}

\subsection{Attention Visualization} Figure~\ref{fig:Viz} shows the attention vector $\alpha$ over parameters for different generated sentences. 
In this figure, and by analyzing more examples~\footnote{\label{note1}Available in supplementary material and later online.}, we can identify two patterns: 
\begin{itemize}[leftmargin=*]
    \item \textit{AoP} learns to focus on the correct skills (i.e., SQL, BOOK) when API-calls are needed. From the first example in Figure~\ref{fig:Viz}, we can see that the activations in $\alpha$ are consistent with those in the correct attention vector $P$. There are also false positives, in which \textit{AoP} puts too high weights on BOOK when the correct response is plain text that should request more information from the user (i.e., \textit{i can help you with that. when would you like to leave the hotel?}). However, we can notice that this example is, in fact, "almost correct" as triggering a booking API call may also be considered a valid response. Meanwhile, the third example also fails to attend to the correct skill, but, in fact, generates a very fluent and relevant response. This is most likely because the answer is simple and generic.
    \item The attention often focuses on multiple skills not directly relevant to the task. We observe this pattern especially when there are other skill-related entities mentioned in the context or the response. For example, in the second dialog example in Figure~\ref{fig:Viz}, we can notice that \textit{AoP} not only accurately focuses on \textit{taxi} domain, but also has non-negligible activations for \textit{restaurant} and \textit{hotel}. This is because the words ``hotel" and ``restaurant" are both mentioned in the dialogue context and the model has to produce two entities of the same type (i.e. \textit{finches bed and breakfast} and \textit{ask}). 

\end{itemize}

\subsection{Data Pre-Processing}
As mentioned in the main article, we convert MultiWOZ into an end-to-end trainable dataset. This requires to add sql-syntax queries when the system includes particular entities. To do so we leverage two annotations such as the state-tracker and the speech acts. The first is used to generate the a well-formed query, including key and attribute, the second instead to decide when to include the query. More details on the dialogue state-tracker slots and slots value, and the different speech acts can be found in \citep{budzianowski2018multiwoz}. 

A query is create by the slots, and its values, that has been updated in the latest turn. The SQL query uses the following syntax:
\begin{equation}
    \textrm{\textbf{SELECT}} * \textrm{\textbf{FROM }} \textrm{\textit{domain}} \textrm{\textbf{ WHERE }} [\textrm{\textit{slot\_type}}=\textrm{\textit{slot\_value}}]^*  \nonumber
\end{equation}
Similarly for the booking api BOOK the syntax is the following:
\begin{equation}
    \textrm{\textbf{BOOK FROM }} \textrm{\textit{domain}} \textrm{\textbf{ WHERE }} [\textrm{\textit{slot\_type}}=\textrm{\textit{slot\_value}}]^*  \nonumber
\end{equation}
In both cases the slot values are kept as real entities. 

More challenging is to decide when to issue such apis. Speech acts are used to decide by using the "INFORM-DOMAIN" and "RECOMMEND-DOMAIN" tag. Thus any response that include those speech tag will trigger an api if and only if:
\begin{itemize}[leftmargin=*]
    \item there has been a change in the state-tracker from the previous turn
    \item the produced query has never been issued before
\end{itemize}
By a manual checking, this strategy results to be effective. However, as reported by \citep{budzianowski2018multiwoz} the speech act annotation includes some noise, which is reflected also into our dataset. 

The results from the SQL query can be of more that 1K records with multiple attributes. Following \citep{budzianowski2018multiwoz} we use the following strategy:
\begin{itemize}[leftmargin=*]
    \item If no speech act INFORM or RECOMMEND and the number of records are more than 5, we use a special token in the memory $<TM>$.
    \item If no speech act INFORM or RECOMMEND and the number of records are less or equal than 5, we put all the records in memory. 
    \item If any speech act INFORM or RECOMMEND, we filter the records to include based on the act value. Notice that this is a fair strategy, since all the resulting record are correct possible answers and the annotators pick-up on of the record randomly~\citep{budzianowski2018multiwoz}.
\end{itemize} 
 Notice that the answer of a booking call instead, is only one record containing the booking information (e.g. reference number, taxi plate etc.) or "Not Available" token in case the booking cannot made.
 
\subsection{Hyper-parameters and Training}
We used a standard Transformer architecture~\citep{vaswani2017attention} with pre-trained Glove embedding~\citep{pennington2014glove}. For the both Seq2Seq and MoE we use Adam~\citep{kingma2014adam} optimizer with a learning rate of $1\times10^{-3}$, where instead for the Transformer we used a warm-up learning rate strategy as in ~\citep{vaswani2017attention}. In both \textit{AoP} and \textit{AoR} we use an additional transformer layer on top the output of the model. Figure~\ref{fig:moe},\ref{fig:aor},\ref{fig:aop} shows the high level design MoE, AoR and AoP respectively. 
In all the model we used a batch size of 16, and we early stopped the model using the Validation set. All the experiments has been conducted using a single Nvidia 1080ti.

We used a small grid-search for tuning each model. The selected hyper-parameters are reported in Table \ref{Hyper-Paramer}, and we run each experiment 3 times and report the mean and standard deviation of each result.
\begin{table}[ht]
\caption{Hyper-Parameters used for the evaluations.}
\label{Hyper-Paramer}
\resizebox{\textwidth}{!}{
\begin{tabular}{r|cccccccc}
\hline
\multicolumn{1}{l|}{\textit{\textbf{Model}}} & \multicolumn{1}{l}{\textit{\textbf{d}}} & \multicolumn{1}{l}{\textit{\textbf{d$_{model}$}}} & \multicolumn{1}{l}{\textit{\textbf{Layers}}} & \multicolumn{1}{l}{\textit{\textbf{Head}}} & \multicolumn{1}{l}{\textit{\textbf{Depth}}} & \multicolumn{1}{l}{\textit{\textbf{Filter}}} & \multicolumn{1}{l}{\textit{\textbf{GloVe}}} & \multicolumn{1}{l}{\textit{\textbf{Experts}}} \\ \hline
\textit{Seq2Seq} & 100 & 100 & 1 & - & - & - & Yes & - \\ \hline
\textit{TRS} & 300 & 300 & 1 & 2 & 40 & 50 & Yes & - \\ \hline
\textit{MoE} & 100 & 100 & 2 & - & - & - & Yes & 13 \\ \hline
\textit{AoP/AoR}& 300 & 300 & 1 & 2 & 40 & 50 & Yes & 13 \\\hline
\textit{TRS/AoP+U}& 300 & 300 & 6 & 2 & 40 & 50 & Yes & 13 \\\hline
\end{tabular}}
\end{table}

\newpage
\subsection{MWOZ and SMD with Std.}
\begin{table}[ht]
\centering
\resizebox{\textwidth}{!}{%
\begin{tabular}{r|cc|cc|cc}
\hline
\multicolumn{1}{c|}{\textbf{Model}} & \textbf{F1} & \textbf{BLEU} & \textbf{SQL$_{Acc}$} & \textbf{SQL$_{BLEU}$} & \textbf{BOOK$_{Acc}$} & \textbf{BOOK$_{BLEU}$} \\ \hline
\textit{Seq2Seq}        & 38.37 $\pm$ 1.69 &	9.42 $\pm$ 0.38    & 49.97 $\pm$ 3.49 &	81.75 $\pm$ 2.54 & 39.05 $\pm$ 9.52 & 79.00 $\pm$ 3.63 \\ \hline
\textit{TRS}            & 36.91 $\pm$ 1.24 &	9.92 $\pm$ 0.43  & 61.96 $\pm$ 3.95 &	89.08 $\pm$ 1.23 & 46.51 $\pm$ 5.46 & 78.41 $\pm$ 2.03 \\ \hline
\textit{MoE}            & 38.64 $\pm$ 1.11 &	9.47 $\pm$ 0.59   & 53.60 $\pm$ 4.62 &	85.38 $\pm$ 2.68 & 37.23 $\pm$ 3.89 & 78.55 $\pm$ 2.62 \\ \hline
\textit{AoR}        & 40.36 $\pm$ 1.39 &	10.66 $\pm$ 0.34 &	69.39 $\pm$ 1.05 &	90.64 $\pm$ 0.83 &	52.15 $\pm$ 2.22 &	81.15 $\pm$ 0.32             \\ \hline
\textit{AoP}  & \textbf{42.26} $\pm$ 2.39 &	\textbf{11.14} $\pm$ 0.39  &\textbf{71.1} $\pm$ 0.47 &	\textbf{90.90} $\pm$ 0.81 & \textbf{56.31} $\pm$ 0.46 & \textbf{84.08} $\pm$ 0.99 \\ \hline \hline
\textit{TRS + U}        & 39.39 $\pm$ 1.23 &9.29	$\pm$  0.71  & 61.80$\pm$ 4.82 &89.70	$\pm$ 1.40& 50.16$\pm$1.18 & 79.05$\pm$ 1.42 \\ \hline
\textit{AoP + U}        & \textbf{44.04} $\pm$ 0.92 &	\textbf{11.26} $\pm$ 0.07 &	\textbf{74.83} $\pm$ 0.79 &	\textbf{91.90} $\pm$ 1.03 &	\textbf{56.37} $\pm$ 0.92 &	\textbf{84.15} $\pm$ 0.32             \\ \hline \hline
\textit{AoP w/o $\mathcal{L}_{V}$}        & 38.50 $\pm$ 1.15 &	10.50 $\pm$ 0.55    & 61.47 $\pm$ 0.15 &	88.28 $\pm$ 0.50 & 52.61 $\pm$ 0.56 & 80.34 $\pm$ 0.21 \\ \hline
\textit{AoP+O}   & \textit{46.36} $\pm$ 0.92 &	\textit{11.99} $\pm$ 0.03  & \textit{73.41} $\pm$ 0.59 &	\textit{93.81} $\pm$ 0.16 & \textit{56.18}$\pm$ 1.55 & \textit{86.42} $\pm$ 0.92 \\ \hline
\end{tabular}
}
\end{table}

\subsection{Persona Result with Std}
\centering
\begin{tabular}{r|cccc}
\hline

\textbf{Model} & \textbf{Ppl.} & \textbf{F1} & \textbf{C} & \textbf{BLEU} \\ \hline
\textit{Seq2Seq}         & 39.42 $\pm$ 1.54 &	6.33 $\pm$ 0.58 &	0.11 $\pm$ 0.06 &	2.80 $\pm$ 0.09 \\ \hline
\textit{TRS}             & 43.12 $\pm$ 1.46 &	7.00 $\pm$ 0.00 &	0.07 $\pm$ 0.16 &	2.56 $\pm$ 0.07 \\ \hline
\textit{MoE}             & \textbf{38.63} $\pm$ 0.20 &   \textbf{7.33} $\pm$ 0.05 &	0.19 $\pm$ 0.16 &	2.92 $\pm$ 0.48 \\ \hline
\textit{AoR}             & 40.18 $\pm$ 0.74 &	6.66 $\pm$ 0.05 &	0.12 $\pm$ 0.14 &	2.69 $\pm$ 0.34 \\ \hline
\textit{AoP}             & 39.14 $\pm$ 0.48 &	7.00 $\pm$ 0.00 &	\textbf{0.21} $\pm$ 0.05 &	\textbf{3.06} $\pm$ 0.08 \\ \hline \hline

\textit{TRS + U}         & 43.04$\pm$ 1.78 &	\textbf{7.33}$\pm$ 0.57 & 0.15$\pm$ 0.02 &	2.66$\pm$0.43  \\ \hline
\textit{AoP + U}         & \textbf{37.40}$\pm$0.08  &	7.00$\pm$0.00  &	\textbf{0.29}$\pm$ 0.07 &	\textbf{3.22}$\pm$ 0.04 \\ \hline  \hline
\textit{AoP w/o $\mathcal{L}_{V}$}       & 42.81$\pm$0.01  &	6.66$\pm$0.57  &	0.12$\pm$ 0.04 &	2.85$\pm$ 0.21 \\ \hline
\textit{AoP + O}         & \textit{40.16} $\pm$ 0.56 &	\textit{7.33} $\pm$ 0.58 &	\textit{0.21} $\pm$ 0.14 &	\textit{2.98} $\pm$ 0.05 \\ \hline
\end{tabular}

\subsection{Domain F1-Score}
\begin{table}[ht]
\centering
\begin{tabular}{rccccccc}
\hline
\multicolumn{1}{c|}{\textbf{Sentence}} & \textbf{Seq2Seq} & \textbf{MoE} & \textbf{TRS} & \textbf{AoR} &  \multicolumn{1}{c|}{\textbf{AoP}} & \textbf{Aop+O} \\ \hline
\multicolumn{1}{r|}{\textit{Taxi}} & 71.77 & 75.97 & 73.92 & 76.07  & \multicolumn{1}{c|}{\textbf{76.58}} & 78.30 \\ \hline
\multicolumn{1}{r|}{\textit{Police}} & 49.73 & 49.95 & 50.24 & 51.95  & \multicolumn{1}{c|}{\textbf{56.61}} & 52.05 \\ \hline
\multicolumn{1}{r|}{\textit{Restaurant}} & 50.20 & 49.59 & 48.34 & \textbf{50.58}  & \multicolumn{1}{c|}{50.47} & 50.90 \\ \hline
\multicolumn{1}{r|}{\textit{Hotel}} & \textbf{46.82} & 45.37 & 43.38 & 45.51  & \multicolumn{1}{c|}{46.40} & 44.47 \\ \hline

\multicolumn{1}{r|}{\textit{Attraction}} & 37.87 & 35.21 & 33.10 & 36.97  & \multicolumn{1}{c|}{\textbf{38.79}} & 37.51\\ \hline
\multicolumn{1}{r|}{\textit{Train}} & 46.02 & 41.72 & 41.28 & 44.33  & \multicolumn{1}{c|}{\textbf{46.32}} & 45.93 \\ \hline
\multicolumn{1}{r|}{\textit{Weather}} & 40.38 & 27.06 & 18.97 & 44.77  & \multicolumn{1}{c|}{\textbf{51.94}} & 55.23 \\ \hline
\multicolumn{1}{r|}{\textit{Schedule}} & 35.98 & 43.94 & 38.95 & 32.90 & \multicolumn{1}{c|}{\textbf{54.18}} & 52.99 \\ \hline
\multicolumn{1}{r|}{\textit{Navigate}} & 18.57 & \textbf{21.34} & 6.96 & 12.69  & \multicolumn{1}{c|}{12.18} & 16.56 \\ \hline
\multicolumn{7}{c}{\textbf{BOOK}} \\ \hline
\multicolumn{1}{r|}{\textit{Taxi}} & 23.16 & 32.28 & 30.70 & 41.93  & \multicolumn{1}{c|}{\textbf{46.66}} & 43.15 \\ \hline
\multicolumn{1}{r|}{\textit{Restaurant}} & 45.02 & 28.26 & 49.72 & 55.70  & \multicolumn{1}{c|}{\textbf{58.51}} & 57.70 \\ \hline
\multicolumn{1}{r|}{\textit{Hotel}} & 49.22 & 31.48 & 41.61 & 51.46 &  \multicolumn{1}{c|}{\textbf{56.62}} & 57.41 \\ \hline
\multicolumn{1}{r|}{\textit{Train}} & 55.86 & 56.38 & 57.51 & 57.51  & \multicolumn{1}{c|}{\textbf{59.15}} & 60.80 \\ \hline
\multicolumn{7}{c}{\textbf{SQL}} \\ \hline
\multicolumn{1}{r|}{\textit{Police}} & 81.33 & 0.00 & 90.66& 76.00  &  \multicolumn{1}{c|}{\textbf{93.33}} & 100.0 \\ \hline
\multicolumn{1}{r|}{\textit{Restaurant}} & 71.58 & 68.00 & 75.90 & \textbf{81.27}  & \multicolumn{1}{c|}{80.43} & 84.15 \\ \hline
\multicolumn{1}{r|}{\textit{Hospital}} & 62.22 & 15.55 & 58.89 & 71.11 &  \multicolumn{1}{c|}{\textbf{76.67}} & 83.33 \\ \hline
\multicolumn{1}{r|}{\textit{Hotel}} & 45.25 & 42.09 & 48.61 & 56.69 &  \multicolumn{1}{c|}{\textbf{59.75}} & 63.75 \\ \hline
\multicolumn{1}{r|}{\textit{Attraction}} & 65.48 & 67.69 & 65.91 & 70.61 & \multicolumn{1}{c|}{\textbf{76.22}} & 74.93 \\ \hline
\multicolumn{1}{r|}{\textit{Train}} & 30.02 & 41.01 & 55.67 & 66.61  & \multicolumn{1}{c|}{\textbf{67.34}} & 69.50 \\ \hline
\end{tabular}
\caption{Per Domain F1 Score.}
\end{table}

\begin{figure}[t]
    \centering
    \includegraphics[width=0.98\linewidth]{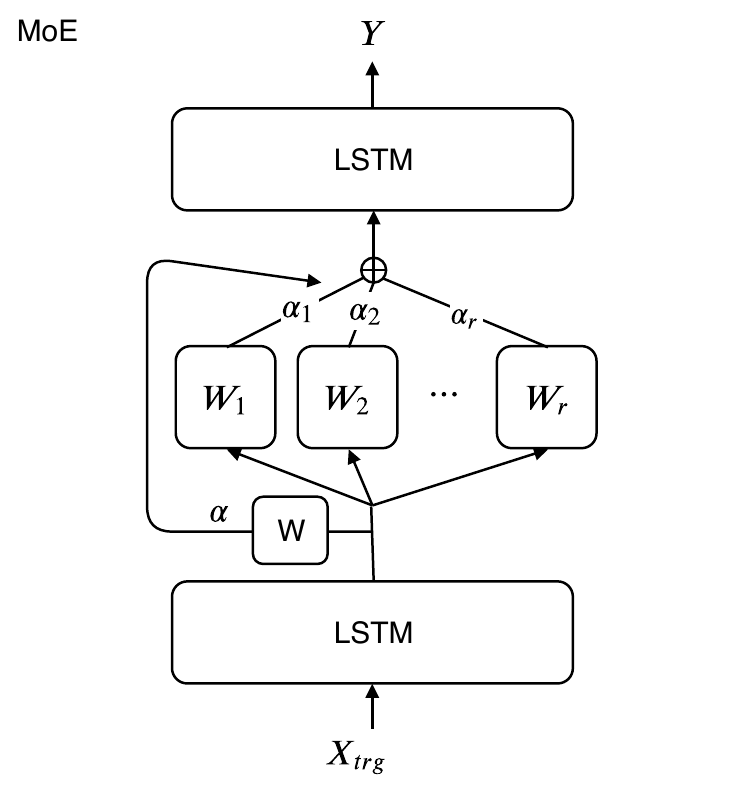}
    \caption{Mixture of Experts (MoE)~\citep{shazeer2017outrageously} model consist of $r$ feed-forward neural network (experts) which are embedded between two LSTM layers, a trainable gating network to select experts.}
    \label{fig:moe}
\end{figure}

\begin{figure}[t]
    \centering
    \includegraphics[width=0.98\linewidth]{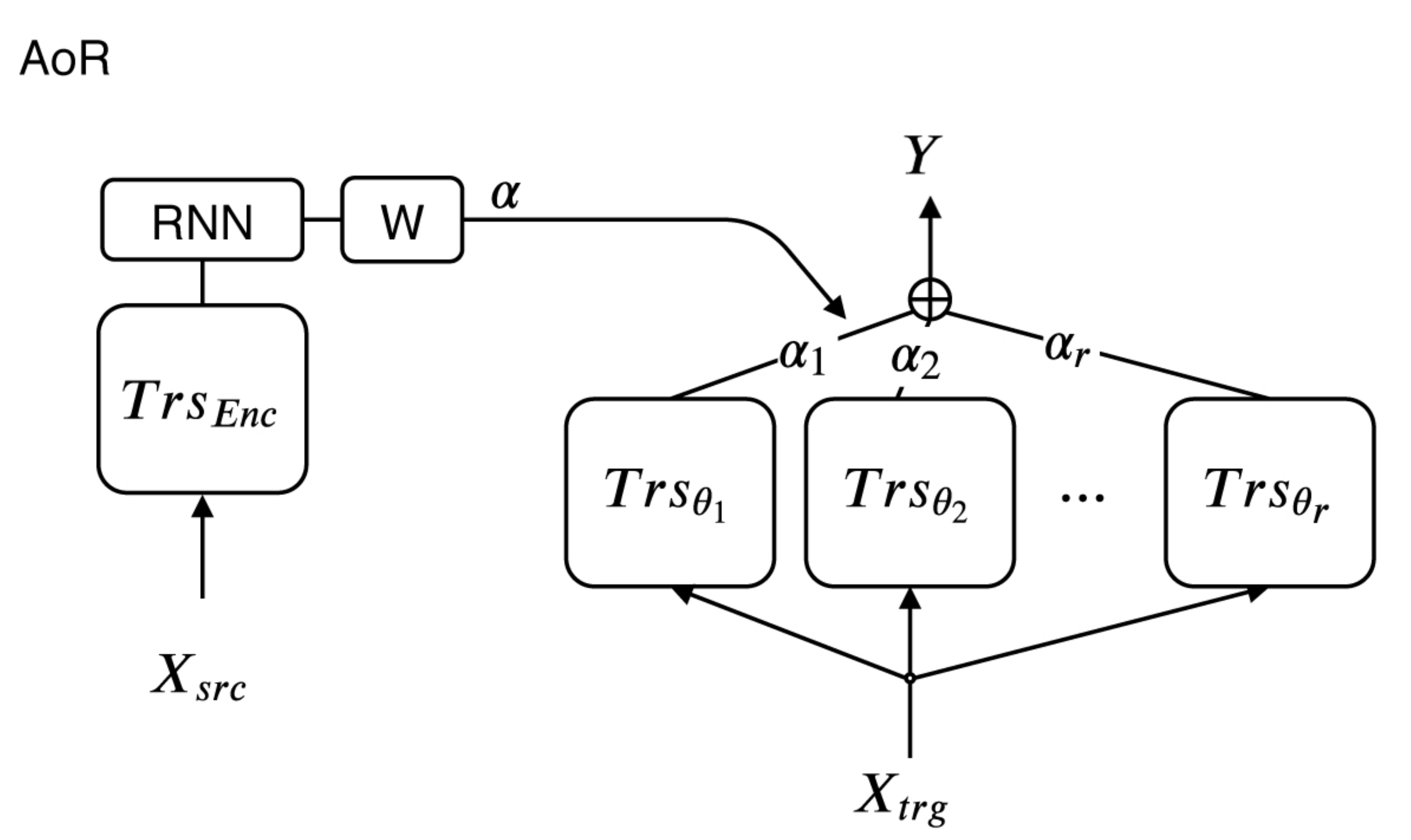}
    \caption{Attention over Representation (AoR) consist of a transformer encoder which encode the source input and compute the attention over the skills. Then $r$ transformer decoder layers computes $r$ specialized representation and the output response is generated based on the weighted sum the representation. In the figure, we omitted the output layer.}
    \label{fig:aor}
\end{figure}

\begin{figure}[t]
    \centering
    \includegraphics[width=0.98\linewidth]{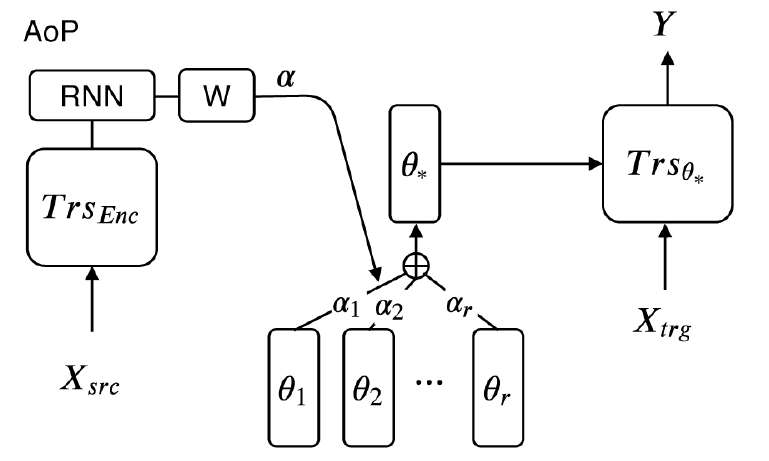}
    \caption{Attention over Parameters (AoP) consist of a transformer encoder which encode the source input and compute the attention over the skills. Then, $r$ specialized transformer decoder layers and a dummy transformer decoder layer parameterized by the weighted sum of the $r$ specialized transformer decoder layers parameters. In the figure, we omitted the output layer.}
    \label{fig:aop}
\end{figure}

\end{document}